\documentclass[fleqn,10pt]{wlscirep}
\PassOptionsToPackage{dvipsnames}{xcolor}
\usepackage[utf8]{inputenc}
\usepackage[T1]{fontenc}

\usepackage{graphicx}%
\usepackage{nameref}
\usepackage{hyperref}
\usepackage{multirow}%
\usepackage{amsthm}%
\usepackage{mathrsfs}%
\usepackage[title]{appendix}%
\usepackage{booktabs}%
\usepackage{algorithm}%
\usepackage{algpseudocode}%
\usepackage{listings}%
\usepackage[utf8]{inputenc} 
\usepackage[T1]{fontenc}    
\usepackage{url}            
\usepackage{nicefrac}       
\usepackage{microtype}      
\usepackage{bm}
\usepackage{makecell}
\usepackage[numbers]{natbib}
\usepackage{epstopdf}
\usepackage{epsfig}
\usepackage[caption=false,font=footnotesize]{subfig}
\usepackage{float}

\def\mf{\mathbf}
\def\mb{\mathbb}
\def\mc{\mathcal}
\def\beq{\begin{equation*}}
\def\eeq{\end{equation*}}
\def\bql{\begin{equation}}
\def\eql{\end{equation}}
\def\bqn{\begin{eqnarray*}}
\def\eqn{\end{eqnarray*}}
\def\bnl{\begin{eqnarray}}
\def\enl{\end{eqnarray}}
\def\bna{\begin{array}{rcl}}
\def\ena{\end{array}}
\def\bnn{\begin{equation}\begin{array}{rcl}}
\def\enn{\end{array}\end{equation}}
\def\bma{\begin{bmatrix}}
\def\ema{\end{bmatrix}}
\def\bmx{\begin{matrix}}
\def\emx{\end{matrix}}
\def\ben{\begin{enumerate}}
\def\een{\end{enumerate}}
\def\bit{\begin{itemize}}
\def\eit{\end{itemize}}
\def\bei{\begin{itemize}}
\def\eei{\end{itemize}}
\def\bet{\begin{tabular}}
\def\eet{\end{tabular}}
\def\sp{\text{span}}

\newcommand{\allcaps}[1]{\uppercase\expandafter{#1}}

\providecommand{\abs}[1]{\lvert#1\rvert}

\usepackage{comment}

\theoremstyle{thmstyleone}%
\newtheorem{theorem}{Theorem}

\theoremstyle{thmstyletwo}%
\newtheorem{remark}{Remark}%

\theoremstyle{thmstylethree}%
\newtheorem{corollary}{Corollary}%

\title{Temporally Consistent Koopman Autoencoders for Forecasting Dynamical Systems}

\author[1,2,+,a]{Indranil Nayak}
\author[2,+]{Ananda Chakrabarti}
\author[1,3]{Mrinal Kumar}
\author[1,2]{Fernando L. Teixeira}
\author[2,3,*]{Debdipta Goswami}
\affil[1]{ElectroScience Laboratory, The Ohio State University, Columbus, OH, 43212, USA}
\affil[2]{Department of Electrical and Computer Engineering, The Ohio State University, Columbus, OH, 43210, USA}

\affil[3]{Department of Mechanical and Aerospace Engineering, The Ohio State University, Columbus, OH, 43210, USA}

\affil[*]{Corresponding author, email: \texttt{goswami.78@osu.edu}}

\affil[+]{These authors contributed equally to this work.}

\affil[a]{Current affiliation: SLAC National Accelerator Laboratory, Stanford University, Menlo Park, CA 94025, USA.}

\begin{abstract}
Absence of sufficiently high-quality data often poses a key challenge in data-driven modeling of high-dimensional spatio-temporal dynamical systems. Koopman Autoencoders (KAEs) harness the expressivity of deep neural networks (DNNs), the dimension reduction capabilities of autoencoders, and the spectral properties of the Koopman operator to learn a reduced-order feature space with simpler, linear dynamics. However, the effectiveness of KAEs is hindered by limited and noisy training datasets, leading to poor generalizability. To address this, we introduce the temporally consistent Koopman autoencoder (tcKAE), designed to generate accurate long-term predictions even with limited and noisy training data. This is achieved through a consistency regularization term that enforces prediction coherence across different time steps, thus enhancing the robustness and generalizability of tcKAE over existing models. We provide analytical justification for this approach based on Koopman spectral theory and empirically demonstrate tcKAE's superior performance over state-of-the-art KAE models across a variety of test cases, including simple pendulum oscillations, kinetic plasma, and fluid flow data.
\end{abstract}
\begin{document}

\flushbottom
\maketitle
%
%
\thispagestyle{empty}


\section*{Introduction}

Long-term forecasting of sequential time-series data generated from a nonlinear dynamical system is a central problem in engineering. For classification or prediction of such time-series data, a precise analysis of the underlying dynamical system provides valuable insights to generate appropriate features and interpretability to any computational algorithm. This gives rise to a family of physics-constrained learning (PCL) algorithms that incorporate constraints arising from physical consistency of the governing dynamical system.

Recently, a physics constrained approach for time-series learning based on Koopman methods \citep{Rowley2009} has been introduced. This approach uses an infinite-dimensional linear operator that encodes a nonlinear dynamics completely. The Koopman operator \citep{Koopman1931} provides a computationally preferable linear method to analyze and forecast dynamical systems. However, Koopman operator maps between infinite-dimensional function spaces, and hence, cannot be \emph{in general} computationally represented without a finite-dimensional projection. Machine learning techniques utilize a finite-dimensional approximation of the Koopman operator by assuming the existence of a finite-dimensional Koopman invariant function space. This is usually achieved by an autoencoder network that maps the state-space into a latent space where the dynamics can be linearly approximated by the finite-dimensional Koopman operator encoded via a linear layer. This Koopman autoencoder (KAE) approach is attractive as it strikes a balance between the expressivity of deep neural networks and interpretability of PCL. Also, Koopman modal approximation can be readily used for stability analysis \citep{Mauroy2016} and control \citep{Otto2020, Goswami2021} of the underlying dynamical system in a data-driven fashion. Recent literature \citep{azencot2020forecasting} also investigates the existence of a backward Koopman operator in order to impose an extra consistency constraint on the latent space linear map, giving rise to consistent Koopman autoencoder (cKAE) algorithm. 

This manuscript proposes a new algorithm for consistent long-term prediction of time-series data generated from a nonlinear (possibly high-dimensional) dynamical system using a temporal consistency constraint with Koopman autoencoder framework. This algorithm, dubbed as temporally-consistent Koopman autoencoder (tcKAE), compares the predictions from different initial time-instances to a final time in the \emph{latent space}. This is in contrast with prior KAE methodologies, where the enforcement of multi-step look-ahead prediction loss is reliant on \emph{labeled} data and/or the latent space representation of the labeled data \citep{azencot2020forecasting,lusch2018deep,takeishi2017learning}. Typically larger look-ahead step during training results in greater prediction accuracy, at the cost of training time. However, having large look-ahead step is a constraint when such data is scarce. The proposed tcKAE is different since it evaluates predictions relative to each other, bypassing direct comparison with the encoded version of the \emph{labeled} data. This alleviates limitations on the maximum look-ahead step imposed by a limited training dataset. However, as discussed in Methods section, selecting arbitrarily large look-ahead steps to enforce consistency among predictions is neither practical nor optimal. Instead, a relatively small number of look-ahead steps can yield comparable or even superior results compared to state-of-the-art Koopman neural network models. The key advantage of the tcKAE algorithm over cKAE \citep{azencot2020forecasting} lies in its ability to operate without assuming the existence of backward dynamics. It is analytically shown that a KAE spanning a Koopman invariant subspace must satisfy the temporal consistency constraint and by enforcing it, tcKAE leads to higher expressivity and generalizability. This enforcement of temporal consistency in effect reduces the sensitivity of the KAE to noise, decreases the variance or uncertainty in future predictions, and make it more robust with less training data as shown in the experiments. Consistency regularization have been used for semi-supervised learning in classification problems \citep{Sajjadi2016, Englesson2021, Fan2022}, but to the best of our knowledge, this is the first attempt to regularize multi-step temporal consistency for dynamical system learning.

\section*{Background}
\textbf{Recurrent neural networks:} The idea of forecasting dynamical systems using neural networks, particularly using recurrent neural networks (RNNs), dates back to several decades~\citep{coryn1998rcurrent,aussem1999dynamical}. Recent advances in big data, machine learning algorithms and computational hardware, have rekindled the interest in this domain. Emergence of innovative architectures such as Long Short-Term Memmory (LSTM)~\citep{hochreiter1997long}, or gated neural networks (GRUs)~\citep{chung2014empirical} have enhanced the prediction capabilities of NN models significantly. However, training an RNN is fraught with vanishing and exploding gradient problems \citep{Bengio1994} which can be solved by approaches like analyzing stability \citep{Miller2019} or using unitary hidden weight matrices \citep{Arjovsky2016}. But these can affect the short term modeling capability by reducing the expressivity of the RNNs \citep{Kerg2019}. Another drawback of RNNs is the lack of generalizability and interpretability especially when used for predicting a physical system. To make RNN physically consistent, a number of methods are presented to take physical constraints into account. These range from making a physics guided architecture \citep{Jia2019}, relating it to dynamical systems \citep{Sussillo2013} or differential equations \citep{chang2019antisymmetricrnn}. Hamiltonian neural networks \citep{greydanus2019hamiltonian} aims to learn the conserved Hamiltonian as a physical constraint but works only for lossless systems. In this paper, we use the linear operator-based recurrent archiecture along with a physical prediction constraint to develop a more generalizable prediction model. 

\noindent
\textbf{Koopman-based methods:} The Koopman operator-based approach exploits the linearity \citep{Rowley2009} in an infinite dimensional function space to yield a linearly recurrent prediction scheme. A number of non-neural dictionary-based approach relying on the dynamic mode decomposition (DMD) \citep{schmid2010} are proposed \citep{Williams2015, Williams2016, Otto2020, Goswami2017, Goswami2021}. However, these methods implicitly assumes a dictionary spanning a Koopman invariant subspace. Neural approaches using Koopman autoencoder (KAE) provide a better alternative to define a Koopman invariant function space where the dynamics can be linearly approximated \citep{takeishi2017learning, otto2019linearly, lusch2018deep}. \citep{lange2021fourier} showcased the superiority of Koopman-based spectral methods over LSTM, Gated Recurrent Unit (GRU), and Echo State Networks (ESNs) due its ability to extract ``slow" frequencies which can have significant effect for long-term predictions. Furthermore, the autoencoder architecture is particularly well-suited (order reduction) for modeling high-dimensional physical systems, the primary goal of this work. However, most of the KAE models just focus on multi-step prediction errors and without any test on consistency of predictions. Recent work \citep{azencot2020forecasting} on backward Koopman operator provides a consistency test by making the predictions forward and backward consistent, but works only when a backward dynamics is well-defined.

\noindent
\textbf{Koopman theory: an overview:}
Consider a discrete-time dynamical system on a $N_d$-dimensional compact manifold $\mathcal{M}$, evolving according to the flow-map $\mf{f}:\mc{M}\mapsto \mc{M}$: 
\begin{equation} \label{Eq: Dynamics}
    \mf{x}_{n+1} = \mf{f}(\mf{x}_{n}),\quad \mf{x}_n\in\mc{M},\quad n\in\mb{N}\cup \{0\}.
\end{equation}
Let $\mc{F}$ be a Banach space of complex-valued observables $\psi:\mc{M}\rightarrow \mb{C}$. The discrete-time \emph{Koopman operator} $\mc{K}:\mc{F}\rightarrow \mc{F}$ is defined as
\begin{equation}
    \mc{K}\psi(\cdot) = \psi \circ \mf{f}(\cdot),\quad \text{with}~~\psi(\mf{x}_{n+1})=\mc{K}\psi(\mf{x}_{n})
\end{equation}
where $\mc{K}$ is infinite-dimensional, and linear over its argument. The scalar observables $\psi$ are referred to as the Koopman observables. Koopman eigenfunctions $\phi$ are special set of Koopman observables that satisfy $(\mc{K}\phi)(\cdot)=\lambda \phi(\cdot)$, with an eigenvalue $\lambda\in \mathbb{C}$. Considering the Koopman eigenfunctions span the Koopman observables,  a {vector valued observable $\mf{g}\in\mc{F}^p=[\psi_1~\psi_2~\ldots~\psi_p]^\text{T}$} can be expressed as a sum of Koopman eigenfunctions $\mf{g}(\cdot)=\sum_{i=1}^{\infty}\phi_i(\cdot)\mf{v}^{\mf{g}}_i$, where $\mf{v}^{\mf{g}}_i\in\mb{R}^p, i=1,2,\ldots,$ are called the \emph{Koopman modes} of the observable $\mf{g}(\cdot)$. This modal decomposition provides the growth/decay rate $\abs{\lambda_i}$ and frequency $\angle{\lambda_i}$ of different Koopman modes via its time evolution
\begin{equation}\label{eq:koop_decomp}
    \mf{g}(\mf{x}_t) = \sum_{i=1}^{\infty}\lambda_i^t\phi_i(\mf{x}_0)\mf{v}^{\mf{g}}_i.
\end{equation}
\noindent The Koopman eigenvalues and eigenfunctions are properties of the dynamics only, whereas the Koopman modes depend on the observable. 

Koopman modes can be analyzed to understand the dominant characteristics of a complex dynamical system and getting traction in fluid mechanics \citep{Rowley2009}, plasma dynamics \citep{nayak2021koopman, nayak2023koopman}, control systems \citep{Otto2020}, unmanned aircraft systems \citep{narayanan2023}, and traffic prediction \citep{Avila2020}. In addition, it is also being used for machine learning tasks and training deep neural networks \citep{Dogra2020}. Several methods have also been developed to compute the Koopman modal decomposition, e.g., DMD and EDMD \citep{schmid2010, Williams2015}, Ulam-Galerkin methods, and Koopman autoencoder (KAE) deep neural networks \citep{otto2019linearly, Yeung2019}. In this paper, we primarily focus on long-term prediction of high-dimensional autonomous dynamical systems using Koopman modes with autoencoder networks.

Temporally-consistent Koopman Autoencoders (tcKAE) is a KAE neural network architecture developed for long-term prediction of high-dimensional nonlinear dynamical systems, primarily aimed for data-scarce scenarios. It enforces the temporal-consistency among predictions across different time-steps. This consistency requirement enhances the robustness of the resultant algorithm against noise, reduces variance, improves the prediction accuracy, and reduce the amount of training data required for similar level of prediction accuracy.

\section*{Methods}

\subsection*{Prediction via Koopman invariant subspace: Koopman autoencoders (KAE)} 
Koopman operator, being an infinite-dimensional one, must be projected onto a finite-dimensional basis for any practical prediction algorithm. The equation in \eqref{eq:koop_decomp} can be truncated to finite set of terms in the presence of a finite-dimensional Koopman invariant subspace. However, discovering such finite-dimensional Koopman invariant subspaces purely from data is a challenging task, and an active area of research. Recent works \citep{azencot2020forecasting, takeishi2017learning,lusch2018deep,otto2019linearly} leverage the nonlinear function approximation capabilities of neural networks to find a suitable transformation from the state space $\mc{M}$ to a Koopman invariant subspace via a latent state $(\mf{z})$. These models are encompassed within the broader umbrella of Koopman autoencoders (KAE) which strives to find a suitable transformation from state space to the Koopman observable space where the dynamics is linear, and can be easily learned. The basic operation of KAE can be decomposed into three components \eqref{eq:KAE_operation}, \textit{i)} encoding $(\Psi_e(\cdot)):\mc{M} \rightarrow \mb{R}^{N_l}$ that transforms the original state $\mf{x}$ to a Koopman observable $\mathbf{z}\in\mb{R}^{N_l}$, \textit{ii)} advancing the dynamics in that transformed space through a linear operator $(K\in\mb{R}^{N_l\times N_l})$, and \textit{iii)} decoding $({\Psi}_d(\cdot))$ back to the original state space:
\begin{equation}\label{eq:KAE_operation}
    \mathbf{z}_n\approx \Psi_e(\mathbf{x}_n)\quad\rightarrow \quad \mathbf{z}_{n+k}=K^k\cdot \mathbf{z}_n \quad \rightarrow \quad \mf{x}_{n+k} \approx \Psi_d(\mf{z}_{n+k})=\hat{\mf{x}}_{n+k},   
\end{equation}

\begin{figure} [htbp]
    \centering
      \includegraphics[width=0.7\linewidth]{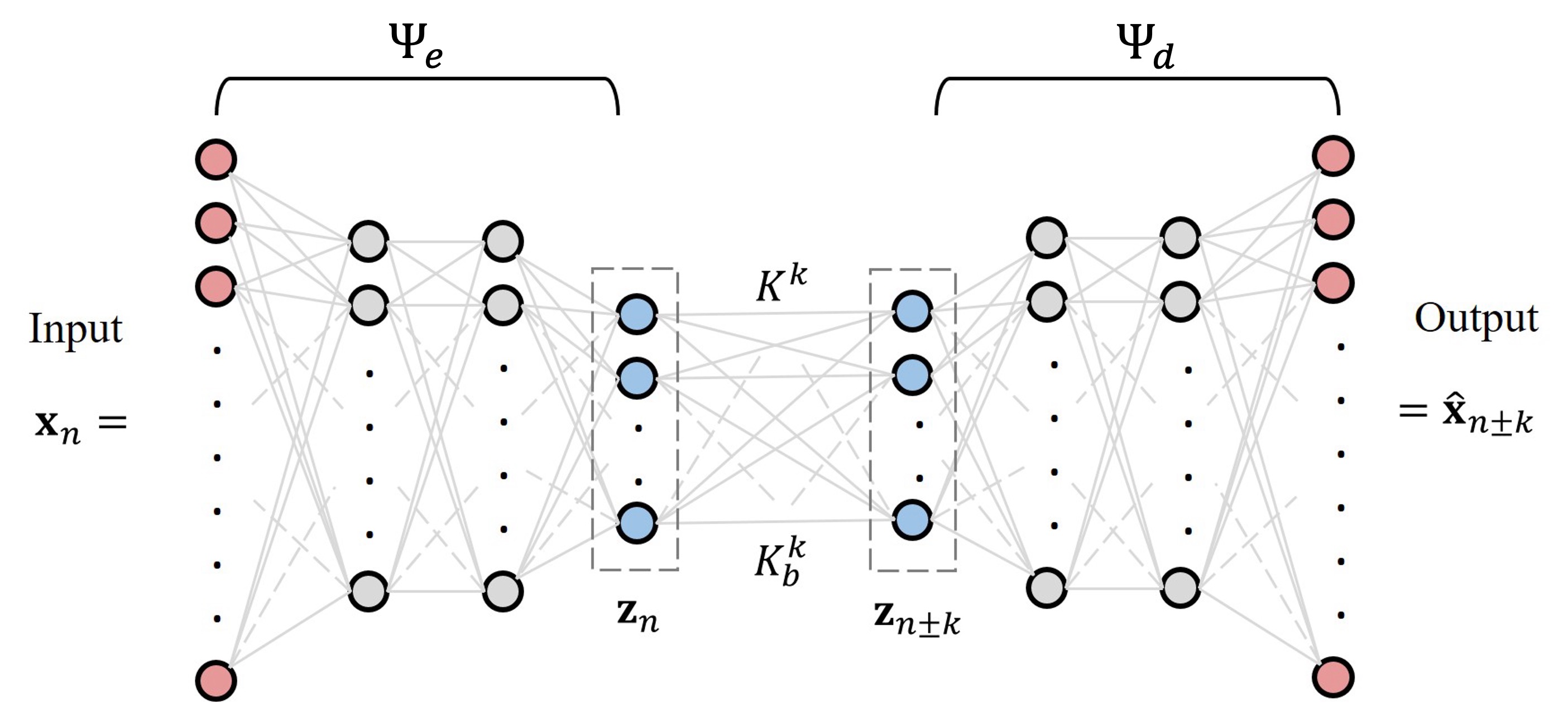}
  \caption{\small{cKAE architecture where the layers with red nodes indicate the input and output layer in the original state space. The latent space is represented by the bottleneck layer with blue nodes. The hidden layers are shown by the gray nodes. The encoder and decoder layers are denoted by $\Psi_e$ and $\Psi_d$ respectively. The linear layers represent the linear forward $(K)$ and backward dynamics $(K_b)$.}}
   \label{fig:KAE_arch}
\end{figure}
where $K$ is the finite-dimensional restriction of $\mc{K}$, which advances $\mathbf{z}$ by $k$ time-steps (or time samples), and $\hat{\mf{x}}$ denotes the KAE reconstruction of the state. Note that for traditional autoencoders the encoding operation typically result in dimensionality reduction which is beneficial when dealing with high-dimensional dynamical systems. Traditionally, since we are mostly interested in future prediction of the state, KAE is trained by minimizing the difference between $\hat{\mf{x}}_{n+k}$ and $\mf{x}_{n+k}$, with $k\in\mb{Z}_+\cup \{0\}$. The loss functions traditionally used for training KAEs are the identity loss $\mc{L}_\text{id}$ ($k=0$), and the forward loss $L_\text{fwd}$ ($k=1,2,\ldots$) \citep{azencot2020forecasting}.   
\begin{align}
    \mc{L}_{id}=\frac{1}{2M}\sum_{n=1}^M||\hat{\mathbf{x}}_n-\mathbf{x}_n||_2, \quad   \mc{L}_{fwd}=\frac{1}{2k_mM}\sum_{k=1}^{k_m}\sum_{n=1}^M||\hat{\mathbf{x}}_{n+k}-\mathbf{x}_{n+k}||_2^2,
\end{align}\label{eq:KAE_loss}
where $||\cdot||_2$ is the 2-norm, $M$ is the number of samples over which we want to enforce the loss, and $k_m$ is the maximum value of $k$, i.e. the maximum look-ahead step for multi-step training. Recent work in \citep{azencot2020forecasting} demonstrated that including backward dynamics $(K_b\in\mb{R}^{N_l\times N_l},~\Psi_d\circ K_b^k \circ \Psi_e (\mf{x}_{n})=\hat{\mf{x}}_{n-k})$ as shown in Figure \ref{fig:KAE_arch} leads to better stability resulting in more accurate long-term predictions. The idea is to incorporate a backward loss term $(\mc{L}_\text{bwd})$ with forward-backward consistency loss $(\mc{L}_\text{con})$,
\begin{align}
      \mc{L}_{bwd}&=\frac{1}{2k_mM}\sum_{k=1}^{k_m}\sum_{n=1}^l||\hat{\mathbf{x}}_{n-k}-\mathbf{x}_{n-k}||_2^2\\
    \mc{L}_{con}&=\sum_{i=1}^{N_b}\frac{1}{2i}||K_{bi\star}K_{\star i}-I_{N_l}||_F+\frac{1}{2i}||K_{\star i}K_{bi\star }-I_{N_l}||_F,     
\end{align}\label{eq:cKAE_loss}
where $N_l$ is the latent space dimension, $K_{bi\star}$ is the upper $i$ rows of $K_b$, $K_{\star i}$ is the $i$ left most columns of $K$, $||\cdot||_F$ denotes the Frobenius norm, and $I_{N_l}\in \mb{R}^{N_l\times N_l}$ is the identity matrix of dimension $N_l\times N_l$.

\subsection*{Temporally-consistent Koopman autoencoder (tcKAE)}

In this section, we describe our proposed temporally-consistent Koopman autoencoder (tcKAE) architecture. The key idea, as illustrated in Figure \ref{fig:tcKAE},  is to introduce a consistency regularization term for training KAE in order to enforce consistency among future predictions. The fundamental idea derives from the time-invariance property of the autonomous dynamical systems. In the context of \eqref{Eq: Dynamics}, the temporal consistency can be stated as the direct result of the following:
\begin{align}
    \mf{f}^{k}(\mf{x}_{n})=\mf{f}^{k + \kappa}(\mf{x}_{n-\kappa}),\quad \forall n, k \in \mb{Z}_+\cup \{0\},~~\kappa=1,2,\ldots, n,
\end{align}
where $\mf{f}^{k}=\mf{f}\circ \mf{f}\circ \ldots (k ~ \text{times})\ldots \circ \mf{f}$, with $\circ$ being the composition operator. One key aspect of our approach is that instead of enforcing this consistency in the original state space, we do so in the Koopman invariant subspace (latent space). This helps to avoid computation in high-dimensional space and provides robustness to noise. We justify enforcing such consistency in the latent space by summarizing the desired property of a consistent autoencoder $\Psi_e$ and constructing the architecture of this framework. The autoencoder $\Psi_e(\cdot) = [\psi_1(\cdot),\ldots,\psi_{N_l}(\cdot)]$ defines a vector-valued observable of the state space $\mc{M}$ where each $\psi_i$ is a scalar-valued function. Let $\mc{G}$ be the span of $\{\psi_1,\ldots,\psi_{N_l}\}$. In order for the linear recurrent matrix $K$ to asymptotically approximate the projection of the Koopman operator $\mc{K}$ on $\mc{G}$, the latter must be Koopman invariant, i.e., $\mc{K}\mc{G}\subsetneq \mc{G}$. The Koopman invariance will ensure the linear recurrent structure  
\bql
\hat{\mf{x}}_{n+1} = \Psi_d \circ K \circ \Psi_e (\mf{x}_n).
\eql The proposed architecture introduces a multi-step temporal consistency loss (not applied against the directly encoded version of the labeled data) in the latent space in order to strengthen the Koopman invariance of the latent function space $\mc{G}$. It is important to emphasize that, for calculating temporal consistency, comparisons are made between multiple forward predictions in the latent space. These predictions are not necessarily constrained by the encoded versions of the available training data. In the following, we summarize the desirable temporal consistency of the encoder functions upon which the proposed method is constructed.

\begin{figure} [htbp]
    \centering
      \includegraphics[width=0.95\linewidth]{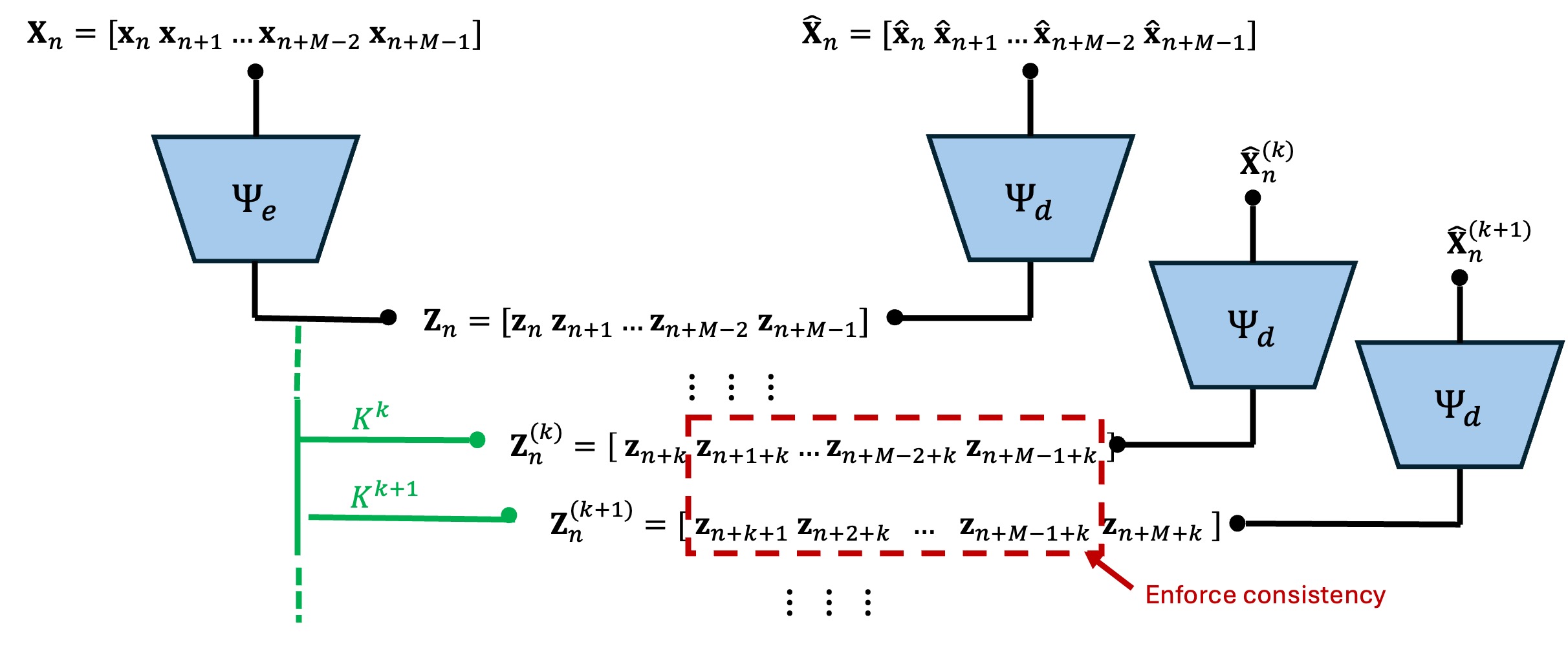}
  \caption{\small{Illustration of enforcement of temporal consistency among predictions in Koopman invariant latent space. }
   \label{fig:tcKAE}}
\end{figure}

\begin{theorem}\label{Thm: pred_consistency}
   Let $\Psi_e(\cdot) = [\psi_1(\cdot),\ldots,\psi_{N_l}(\cdot)]^T \in \mc{F}^{N_l}$ denote a vector valued function (encoder) comprised of scalar functions $\psi_i(\cdot)\in \mc{F}$. The latent space $\mc{G} = \sp\{\psi_1(\cdot)\ldots,\psi_{N_l}(\cdot)\}$ forms a Koopman invariant subspace with respect to the system dynamics \eqref{Eq: Dynamics} if and only if 
there exists $K\in \mb{R}^{N_l\times N_l}$ such that 
\bql \label{Eq: pred_consistency} \Psi_e(\mf{x}_{n+k}) = K^{k}\Psi_e(\mf{x}_{n}) \eql for all $n\geq 0$ and $k\geq 1$.
\end{theorem}

\begin{proof}
    $(\Rightarrow)$ Suppose $\Psi_e(\cdot)$ is defined such that $\mc{G}$ is a Koopman-invariant subspace. We prove this part through induction. For any $n\geq 0$, \eqref{Eq: pred_consistency} is satisfied for $k=1$ with a matrix $K$ where $K=\mc{K}_{\mc{G}}$, i.e., the restriction of the Koopman operator $\mc{K}$ on $\mc{G}$. Now suppose that it is satisfied for a fixed $k\in\mb{N}$. Then, $\Psi_e(\mf{x}_{n+k+1}) = \Psi_e\circ \mf{f}(\mf{x}_{n+k}) = \mc{K}\Psi_e(\mf{x}_{n+k}) = \mc{K}K^{k}\Psi_e(\mf{x}_{n}) = K^{k}\mc{K}\Psi_e(\mf{x}_{n}) = K^{k} \Psi_e\circ \mf{f}(\mf{x}_{n}) =  K^{k}\Psi_e(\mf{x}_{n+1}) = K^{k+1}\Psi_e(\mf{x}_{n})$, where the first two steps are by the definition of Koopman operator, next two steps are by the linearity of the same, and the last step is by the induction anchor.

    $(\Leftarrow)$ Suppose $\exists$ $K\in \mb{R}^{N_l\times N_l}$ such that \eqref{Eq: pred_consistency} is satisfied for all $n\geq 0$ and $k\geq 1$.  Then it must be satisfied for $n=0$ and $k=1$, i.e., $\Psi_e(\mf{x}_{1}) = K\Psi_e(\mf{x}_{0})$. Since $\mf{x}_0$ can be anything in $\mc{M}$, this yields $\Psi_e\circ\mf{f}(\mf{x}) = K\Psi_e(\mf{x})$ for all $\mf{x}\in\mc{M}$, i.e., $\Psi_e\circ\mf{f}(\cdot) = K\Psi_e(\cdot)$.  Let $g\in\mc{G}$ be an observable. Then, $\exists$ $\mf{\alpha} = [\alpha_1,\ldots,\alpha_{N_l}]\in \mb{R}^{N_l}$ such that $g(\cdot) = \sum\limits_{i=1}^{N_l}\alpha_i\psi_{N_i}(\cdot)=\mf{\alpha}^\top \Psi_e(\cdot)$. Now, $\mc{K}g(\cdot) = g\circ\mf{f}(\cdot) = \mf{\alpha}^ \top \Psi_e\circ\mf{f}(\cdot) = \mf{\alpha}^\top K\Psi_e(\cdot) = \sum\limits_{j=1}^{N_l}\left(\sum\limits_{i=1}^{N_l} \alpha_i K_{ij}\right) \psi_j(\cdot)$, i.e., $\mc{K}g\in\mc{G}$. Hence, $\mc{G} = \sp\{\psi_1(\cdot)\ldots,\psi_{N_l}(\cdot)\}$ forms a Koopman-invariant subspace.
\end{proof}
Theorem \ref{Thm: pred_consistency} ensures that, in the latent space, encoded state at any time instant must be equal to the encoded state of any previous time multiplied with $K^{k}$ where $k$ denotes the shift in terms of time samples between those two instances. It also forces the latent state to have linear time-invariant dynamics with a constant state-transition matrix $K$. Theorem \ref{Thm: pred_consistency} provides a method to ensure the autoencoder span a Koopman invariant subspace by enforcing temporal consistency as follows.
\begin{corollary}\label{cor:pcloss}
For a KAE $\Psi_e:\mc{M}\rightarrow \mb{R}^{N_l}$, if the associated function space $\mc{G}$ is Koopman invariant, then for any $n\geq 0$ and $k\geq 1$, $K^{k-p}\Psi_e(\mf{x}_{n+p}) = K^{k-q}\Psi_e(\mf{x}_{n+q})$ for all $p,q=1,\ldots,k-1$, $(p\neq q)$ i.e., 
\bqn K^{k-1}\Psi_e(\mf{x}_{n+1}) = K^{k-2}\Psi_e(\mf{x}_{n+2}) =\ldots = K\Psi_e(\mf{x}_{n+k-1}).\eqn 
\end{corollary}

Corollary \ref{cor:pcloss} is essentially heart of the temporal consistency loss. In order to define our loss function, let us consider the reference time-step to be $n$. For a batch size of $M$ (in latent space $\mf{Z}_n=[\mf{z}_n~\mf{z}_{n+1}\ldots~\mf{z}_{n+M-1}]$), let the KAE advance the dynamics up to $k_{tm}$ time-steps, giving rise to a set of batches $\mf{Z}^{(k)}_n=[\mf{z}_{n+k}^{(k)}~\mf{z}_{n+1+k}^{(k)}\ldots~\mf{z}_{n+M+k-1}^{(k)}]$ with $k = 1,2,\ldots,k_{tm}$, where the superscript $(k)$ denotes the number of operations of $K$. The consistency loss $\mc{L}_\text{tc}$ (Figure \ref{fig:tcKAE}) can be defined as,
\bql\label{Eq: tc_loss}
    \mc{L}_\text{tc}=\frac{1}{2(k_{tm}-1)}\sum_{q=1}^{k_{tm}-1}\mc{L}_q, \quad   \mc{L}_{q}=\frac{1}{(k_{tm}-q)}\sum_{k=1}^{k_{tm}-q}\frac{1}{(M-q)}\sum_{p=q}^{M-1}||\mf{z}^{(k)}_{n+k+p}-\mf{z}^{(k + q)}_{n+k+p}||_2^2,
\eql
where $\mc{L}_q$ takes into account the common terms between two batches shifted by $q$ time samples. The total loss $(\mc{L}_\text{tot})$ for training tcKAE is given by,
\begin{align}\label{eq:total_loss}   \mc{L}_\text{tot}=\gamma_\text{id}\mc{L}_\text{id} + \gamma_\text{fwd} \mc{L}_\text{fwd} + \gamma_\text{bwd} \mc{L}_\text{bwd} + \gamma_\text{con}\mc{L}_\text{con} + \gamma_\text{tc}\mc{L}_\text{tc},
\end{align}

where $\gamma_\text{id}, \gamma_\text{fwd}, \gamma_\text{bwd}, \gamma_\text{con}$ and $\gamma_\text{tc}$ are the corresponding scaling factors.

\begin{remark}
    The temporal consistency loss enforces a Koopman invariant subspace by virtue of the time-homogeneity of the latent-state dynamics, i.e., the Koopman operator is time-invariant for an autonomous system. This provides a \emph{physics constrained learning} method for dynamical system prediction even from a smaller dataset.
\end{remark}

\subsection*{Simulation Setup}
In this section we discuss the simulation setup in details for both the oscillating electron beam and flow past cylinder.

\begin{figure} [htbp]

 \subfloat[\label{fig:beam_snap_n40000}]{
      \includegraphics[width=0.28\textwidth]{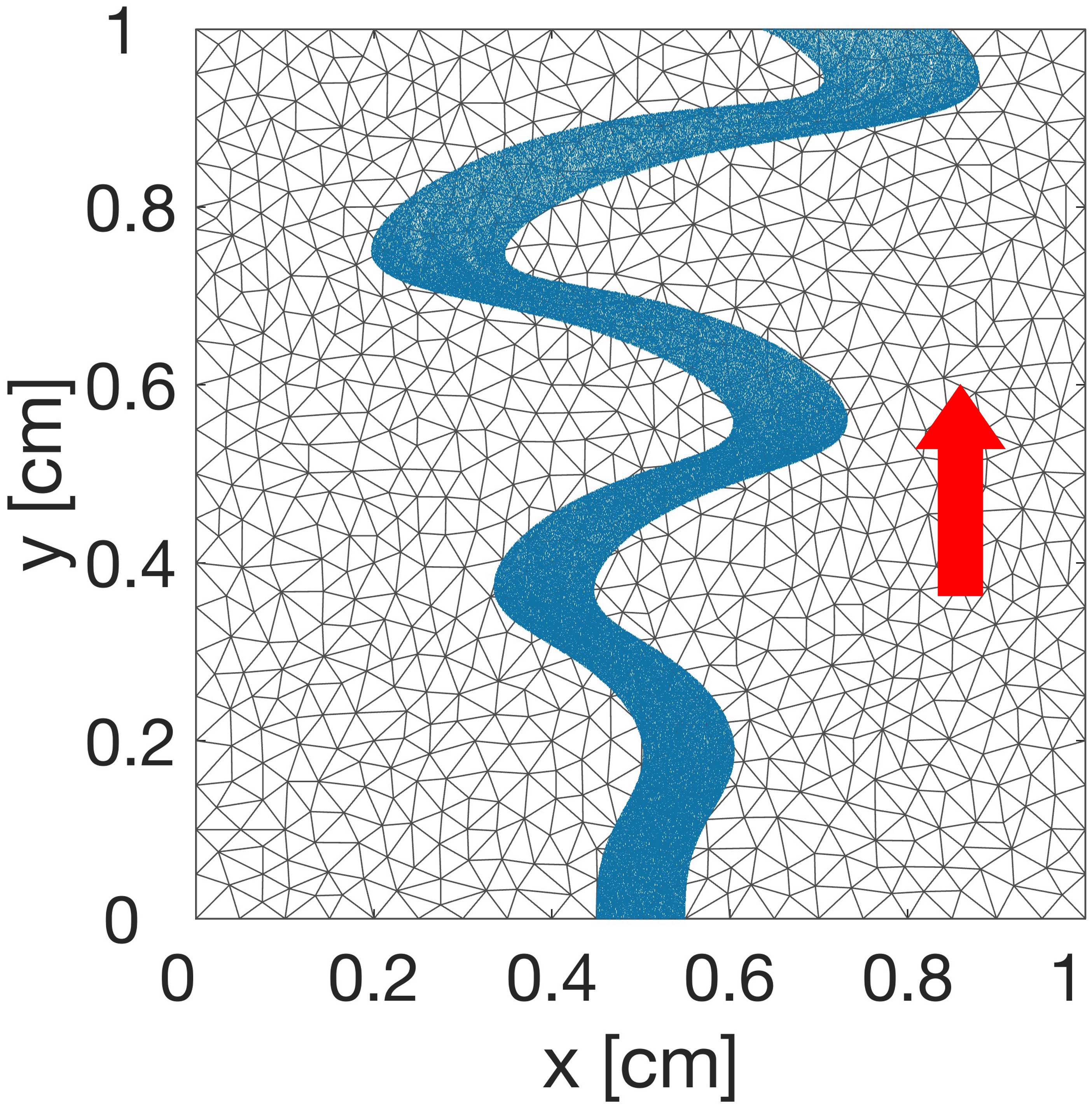}}
 \subfloat[\label{fig:flow_cyl_snap}]{
      \includegraphics[width=0.72\textwidth]{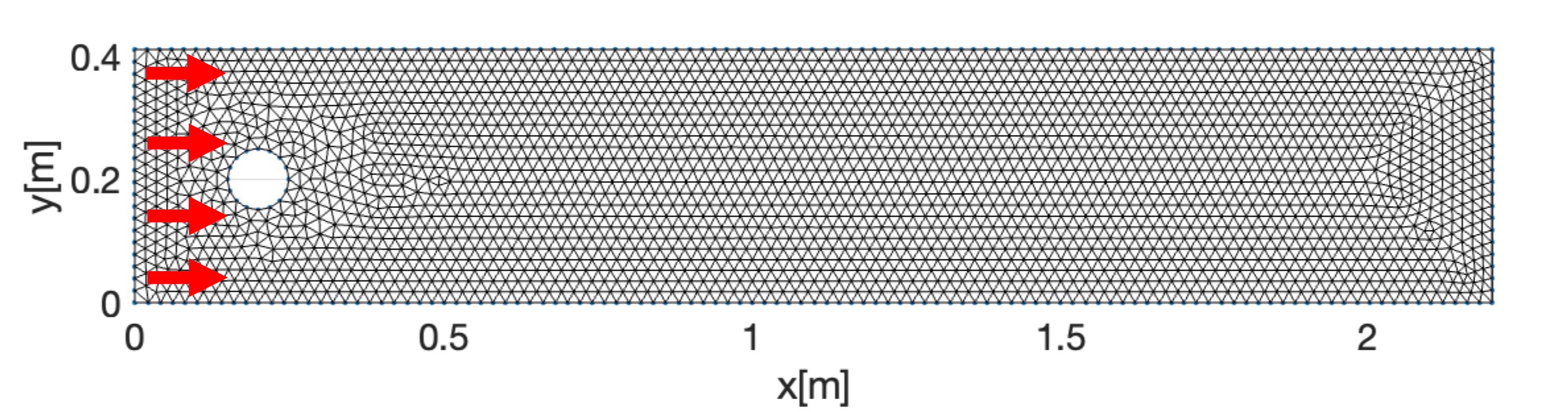}}
  \caption{ Simulation setup for benchmark dynamical systems: (a) Snapshot of the 2D electron beam (blue dots indicates the charged particles) at $t=8$ ns with unstructured triangular mesh. Red arrow indicates the beam direction. ; (b) Irregular triangular mesh for flow-past cylinder finite element simulation. Red arrows indicate the flow direction.}
\end{figure}


\subsubsection*{Oscillating electron beam}
A two-dimensional (2-D) electron beam (Figure \ref{fig:beam_snap_n40000}) is simulated inside a square cavity of dimension $1~\text{cm}\times 1~\text{cm}$ using a charge-conserving electromagnetic particle-in-cell (EMPIC) algorithm \cite{na2016local}. The electron beam is propagating along the $+$ve $y$ direction under the influence of an oscillating transverse magnetic flux. The solution domain is discretized using irregular triangular mesh (grey lines in Figure \ref{fig:beam_snap_n40000}) with $N_0=844$ nodes, $N_1=2447$ edges and $N_2=1604$ elements (triangles). The blue dots in Figure \ref{fig:beam_snap_n40000} represent the superparticles which are discretized representation of the phase-space of electrons (delta distribution in both position and velocity). The superparticles are essentially the point charges with charge $q_{sp}=r_{sp}q_e$, and mass $m_{sp}=r_{sp}m_e$, where $r_{sp}$ is the superparticle ratio with $q_e$ and $m_e$ respectively representing the charge and mass of an electron. We select $r_{sp}=5000$. The superparticles are injected from the bottom of the cavity randomly with uniform distribution over the region $[0.45~\text{cm},~0.55~\text{cm}]$ with the injection rate of $50$ superparticles per time-step. The superparticles are injected along the vertical direction with $v_y=5\times10^6$ m/s. The external oscillating magnetic flux can be represented by $\mathcal{B}_{ext}(t)=\mathcal{B}_0~\text{sin}(2\pi t/T_{osc})~\hat{z}$ with $\mathcal{B}_0=2.5\times10^{-2}$ T, and $T_{osc}=0.8$ ns. The time-step for the EMPIC simulation is taken to be $\Delta t=0.2$ ps.

\subsubsection*{Flow past cylinder}
The 2-D solution domain $(2.2~\text{m} \times 0.41~\text{m}) $ is discretized using irregular triangular mesh (Figure~\ref{fig:flow_cyl_snap}) with number of nodes $N_0=2647$, and number of elements (triangles) $N_2=5016$. The cylinder has the diameter of $0.1$ m with center of the cylinder located at $(0.2~\text{m},0.2~\text{m})$. The flow is assumed to be incompressible, and governed by the Navier-Stokes equations with $\mathbf{u}$, $\mathbf{v}$ denoting the horizontal and vertical component of the velocity respectively while $\mathbf{p}$ denotes the pressure field. The density of the fluid is set to $\rho=1$ Kg/m$^3$, and dynamic viscosity $\mu=0.001$ Kg/m s. The flow is unsteady with a maximum velocity of 1 m/s and mean velocity $\frac{2}{3}$ of the maximum velocity. The simulation starts with initial conditions $\mathbf{u}_0=\mathbf{v}_0=0$ and $\mathbf{p}_0=0$. For the boundary conditions, the leftmost boundary is set as an inlet with a parabolic velocity profile. This is representative of fully developed laminar flow at the inlet. The rightmost boundary is set as an outflow (pressure boundary), where we specify the pressure but do not specify the velocity, allowing the flow to exit naturally based on the internal flow field. All other boundaries are treated as walls with a no-slip condition, which means the fluid velocity at the walls is zero. The simulation runs for a total of 80 seconds, with a time-step size of 0.01 seconds.

\subsection*{Network Architecture \& Training strategy}

The overall KAE network architecture is shown in Figure \ref{fig:KAE_arch} and the tcKAE architecture is shown in Figure \ref{fig:tcKAE}. Note that the network topology is same as that in \cite{azencot2020forecasting}. The number of nodes in the input and output layer (red) is indicated by $N_\text{in}$ and $N_\text{out}$ ($N_\text{in}=N_\text{out}$). $N_l$ denotes the number of nodes in the bottleneck layer (blue), which represents the approximation of Koopman-invariant latent space. Number of nodes in each hidden layers (two hidden layers for each encoder and decoder) is represented by $N_h$ (gray nodes). We keep $N_h$ same for both the hidden layers and it is tuned in multiple of $16$. We use the $tanh$ activation function for each layer except the bottleneck layer.

The experiments were carried out on the Ohio Supercomputer's (OSC)  \cite{ohio_supercomputer} Owens cluster, consisting of 842 Dell Nodes and 648 compute nodes with Dell PowerEdge C6320 two-socket servers, Intel Xeon E5-2680 v4 (Broadwell, 14 cores, 2.40 GHz) processors and 128 GB memory. As shown in Figure \ref{fig:tcKAE_dataset}, we partition the dataset $\mathbf{X}$ into training ($\mathbf{X}_{\text{train}}$), validation ($\mathbf{X}_{\text{val}_1}, \mathbf{X}_{\text{val}_2}$) and testing ($\mathbf{X}_{\text{test}}$) set. For a set of hyperparameters, the tcKAE architecture is trained to reduce the training loss $\mc{L}_\text{tot}$ as defined in \ref{eq:total_loss}. The set of hyperparameters which provides lowest validation loss is chosen to be the optimal set of hyperparameters for each case. The validation loss is computed as the average relative 2-norm error over $\mathbf{X}_{\text{val}_2}$, using predictions initialized from points in $\mathbf{X}_{\text{val}_1}$ as initial conditions. We analyze four different scenarios for each test case with clean and noisy (30 dB SNR) data for two different $N_{\text{train}}$ values. Instead of enforcing $\mc{L}_\text{tc}$ from the very first epoch of training, we enforce it after certain number of epochs $e_s$, which is a hyperparameter for training the tcKAE. That is, we keep $\gamma_{\text{tc}}=0$ for the first $e_s$ number of epochs. Up to $e_s$ we train without enforcing the consistency regularization, and switch to nonzero $\gamma_\text{tc}$ from $e_s$ onward. The look-ahead step $k_{tm}$ for enforcing temporal consistency is an important hyperparameter. Ideally, a larger $k_{tm}$ should lead to better accuracy, since we are looking further into the future. However, as seen in (11), increasing $k_{tm}$ also increases the number of terms in the summation for the temporal consistency loss. Since in practice we are using a numerical stochastic optimization routine, the increasing complexity (and possibly non-convexity) of the loss term increases the possibility of the optimization algorithm to get stuck in a local minima, thereby resulting in higher error and increased training time. As a result, in practice, we typically avoid selecting very large $k_{tm}$ unless there is a significant improvement in accuracy. The weights for different loss components are one of the most crucial hyperparameters. Although, there are no fixed rule for selecting these weights, it is not entirely ad-hoc either. Since we are interested in forward prediction from the trained model, $\gamma_\text{fwd}$ and $\gamma_\text{id}$ are most crucial. In order to reduce the number of variables, we set $\gamma_\text{id}=1$ and vary other weights accordingly. Note that the advantage of tcKAE over DAE and cKAE comes at the cost of extra hyperparameters and increased training time. Training time, detailed training parameters, ablation study, and effect of $k_\text{tm}$ on prediction accuracy is provided in the supplementary material (see Supplementary Tables 1-12). \par

\section*{Results} \label{sec:results}
\subsection*{Baseline}
 The performance of the proposed prediction-consistent Koopman autoencoder (tcKAE) is compared against the state-of-the-art reduced-order KAE forecasting models for high-dimensional physical systems. As discussed earlier, several neural-network based models such as recurrent neural networks (RNNs) \citep{chung2014empirical,kerg2019non,chang2019antisymmetricrnn}, Hamiltonian neural networks (HNNs) \citep{greydanus2019hamiltonian} have already been applied successfully in forecasting dynamical systems. However, these models do not fare well for very long-term predictions \citep{lange2021fourier}, and very high-dimensional systems. Koopman-based autoencoder models \citep{lusch2018deep,otto2019linearly} have been shown to be effective in capturing this long-term behavior, especially for high-dimensional systems. However, recently proposed consistent Koopman autoencoder \citep{azencot2020forecasting} which we will refer to as cKAE, has been shown to outperform all the mentioned methods for long-term predictions in terms of stability. We will consider cKAE as the baseline for benchmarking our proposed tcKAE for several datasets. We will also compare our results with dynamic autoencoders (DAEs), the term used in \citep{azencot2020forecasting} to describe the models in \citep{lusch2018deep,otto2019linearly}. 

\begin{figure} [t]
    \centering
      \includegraphics[width=0.6 \linewidth]{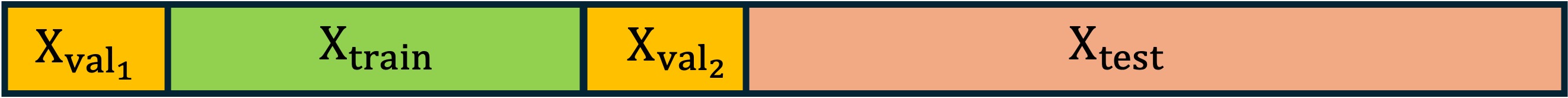}
  \caption{\small{Splitting of time-series dataset for training, validation and testing. }
   \label{fig:tcKAE_dataset}}
\end{figure}

\subsection*{Performance on Benchmark Datasets} \label{sec:dataset} In order to streamline the discussion, we adopt the following notation: the complete dataset is denoted by $\mathbf{X}\in \mathbb{R}^{N_\text{dim}\times N_\text{tot}}$, where $N_\text{dim}$ is the dimension of the state-space, and $N_\text{tot}$ is the total number of time samples~\footnote{Note that we distinguish between time steps and time samples. By time step, we refer to the discrete time instances used by high-fidelity numerical solvers to generate data. This high-fidelity data is typically sampled at specific time steps and used for KAE training.}. We denote the training set as $\mathbf{X}_{\text{train}}\in\mathbb{R}^{N_\text{dim}\times N_\text{train}}$, validation set as $\mathbf{X}_{\text{val}}\in\mathbb{R}^{N_\text{dim}\times N_\text{val}}$, and testing set as $\mathbf{X}_{\text{test}}\in\mathbb{R}^{N_\text{dim}\times N_\text{test}}$ with $N_{\text{train}}$, $N_{\text{val}}$ and $N_{\text{test}}$ being the number of time samples in training, validation and testing set respectively. Note that validation for long-term time-series prediction is nontrivial, and authors in \cite{azencot2020forecasting} does not mention any validation set. However, for practical applications, validation is crucial, and we validate using the following method: We divide the validation set in two sets $\mathbf{X}_{\text{val}_1}\in \mb{R}^{N_\text{dim} \times N_{\text{val}_1}}$ and $\mathbf{X}_{\text{val}_2}\in \mb{R}^{N_\text{dim} \times N_{\text{val}_2}}$, separated by $\mathbf{X}_{\text{train}}$ as shown in Figure \ref{fig:tcKAE_dataset}, where $N_{\text{val}}=N_{\text{val}_1}+N_{\text{val}_2}$. We chose $N_{\text{val}_1}=N_{\text{val}_2}=\frac{1}{4}N_{\text{train}}$ for all the datasets. The datapoints in $\mathbf{X}_{\text{val}_1}$ are used as initial conditions to predict over the span of $\mathbf{X}_{\text{val}_2}$. The error metric we use in this work is the 2-norm relative error defined (at $n^\text{th}$ time sample) by $\delta_n = \frac{||\hat{\mathbf{x}}_n - \mathbf{x}_n||_2}{||\mathbf{x}_n||_2}$, where $\mathbf{x}$ is the original state and $\hat{\mathbf{x}}$ is the predicted state.  The validation loss is determined by calculating the average relative 2-norm error across $\mathbf{X}_{\text{val}_2}$, where predictions are generated using initial conditions from points in $\mathbf{X}_{\text{val}_1}$. The test error is calculated as the 2-norm relative error in the extrapolation region. To generate an error bar, we use the variance quantified by the 90\% confidence interval, defined by the 5\textsuperscript{th} and 95\textsuperscript{th} percentiles of the error. This provides a robust summary of the typical error range by excluding the extreme 5\% values at both ends of the distribution. We follow the statistical analysis approach used in \cite{azencot2020forecasting}. Specifically, we take the first $N_s$ time steps ($N_s = N_\text{train}$) from the testing set and use the corresponding states as initial conditions. The trained model is then used to predict up to a fixed number of time steps. For each initial condition, we compute the 2-norm error by comparing the predictions with the ground truth. These errors are then used to generate the confidence interval. The reported results represent the average of neural network training outcomes across 10 different random seeds.

\begin{figure*} [htbp]

    \centering
  \subfloat[$N_\text{train}=32$, Clean. \label{fig:err_pend_ntr32snr0} ]{%
       \includegraphics[width=0.49\linewidth]{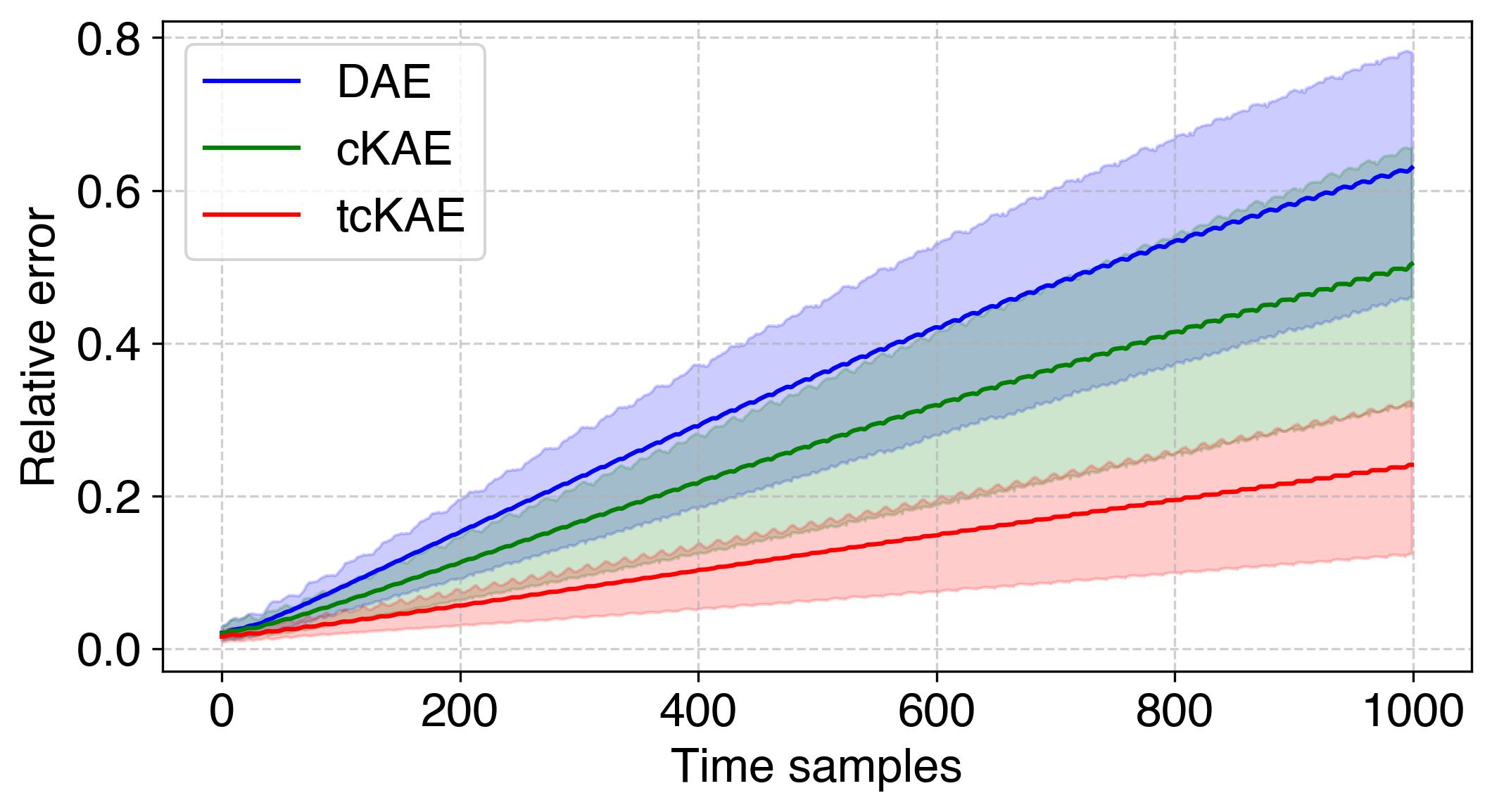}}
    \hfill
  \subfloat[ $N_\text{train}=32$, 30 dB SNR. \label{fig:err_pend_ntr32snr30} ]{%
        \includegraphics[width=0.49\linewidth]{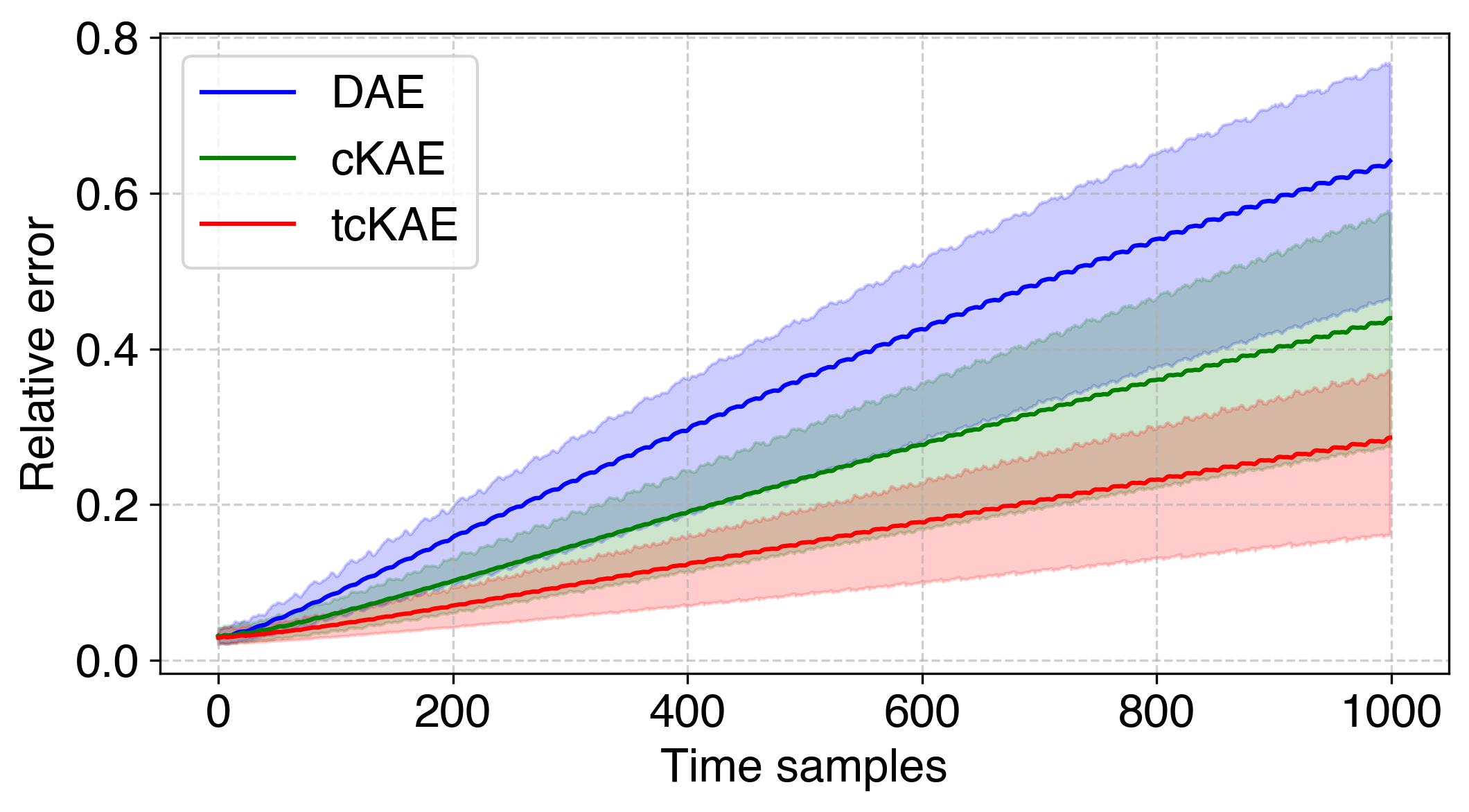}}
    \\ 
   \subfloat[$N_\text{train}=64$, Clean.\label{fig:err_pend_ntr64snr0} ]{%
        \includegraphics[width=0.49\linewidth]{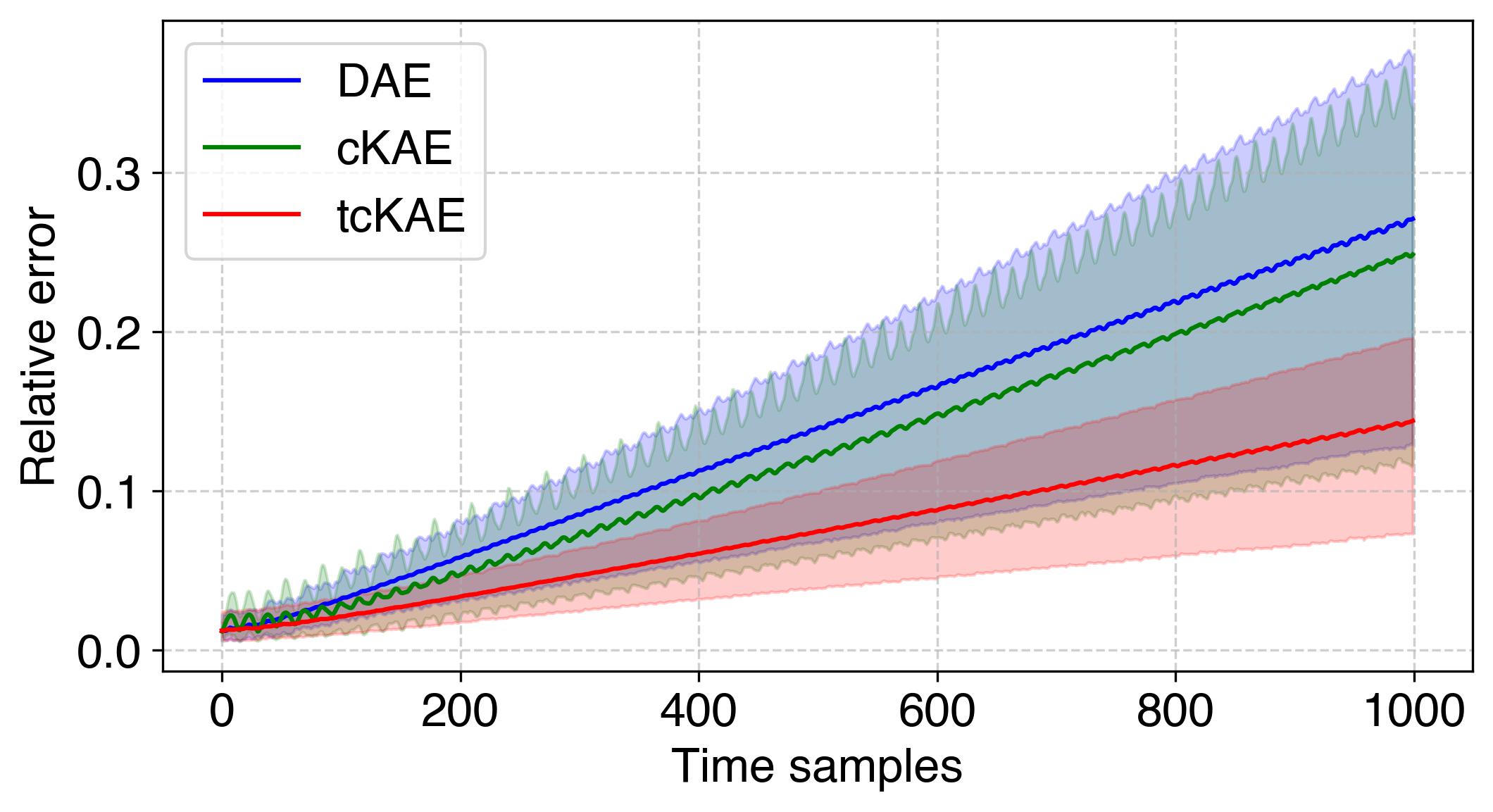}}
    \hfill 
    \subfloat[$N_\text{train}=64$, 30 dB SNR.\label{fig:err_pend_ntr64snr30} ]{%
        \includegraphics[width=0.49\linewidth]{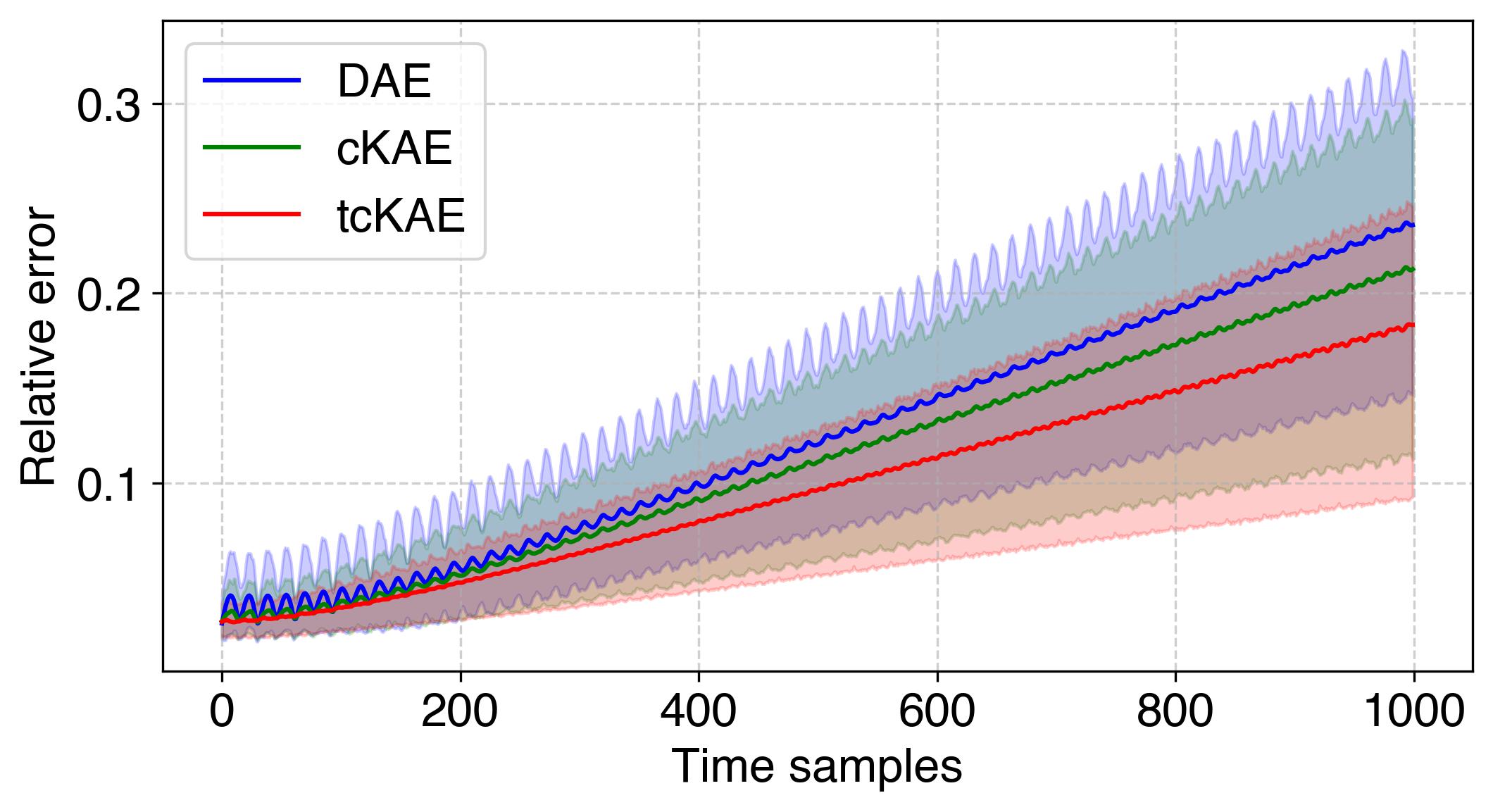}} 
  \caption{\small{Comparison between DAE, cKAE, and tcKAE for long-term extrapolation performance of pendulum dynamics. The shaded portion denotes the 90\% confidence interval.}  }\label{fig:err_pend}
\end{figure*}

\begin{figure*} [t]

    \centering
  \subfloat[$N_\text{train}=28$, Clean. \label{fig:err_wavy_ntr28snr0} ]{%
       \includegraphics[width=0.49\linewidth]{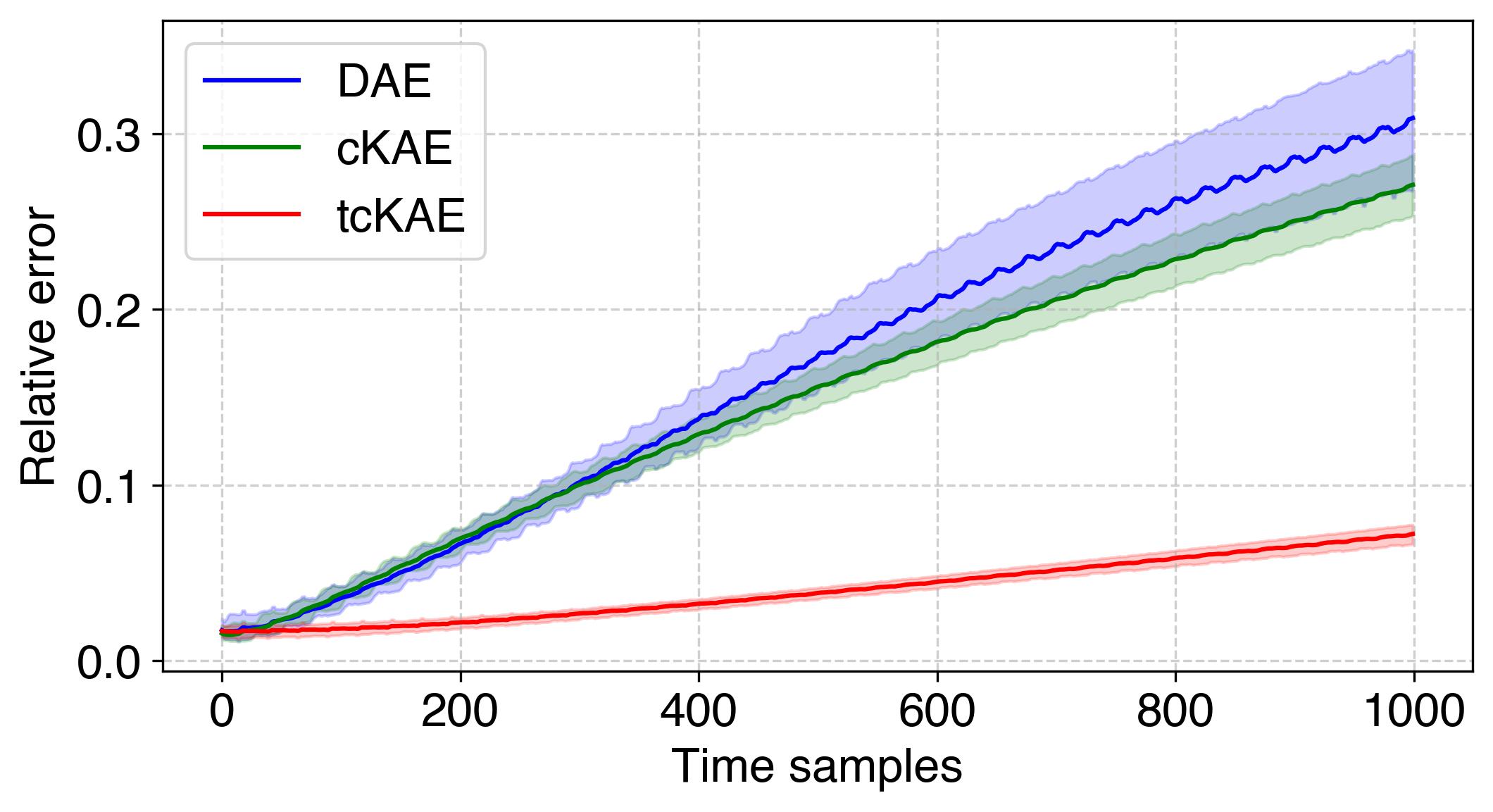}}
    \hfill
  \subfloat[ $N_\text{train}=28$, 30 dB SNR. \label{fig:err_wavy_ntr28snr30} ]{%
        \includegraphics[width=0.49\linewidth]{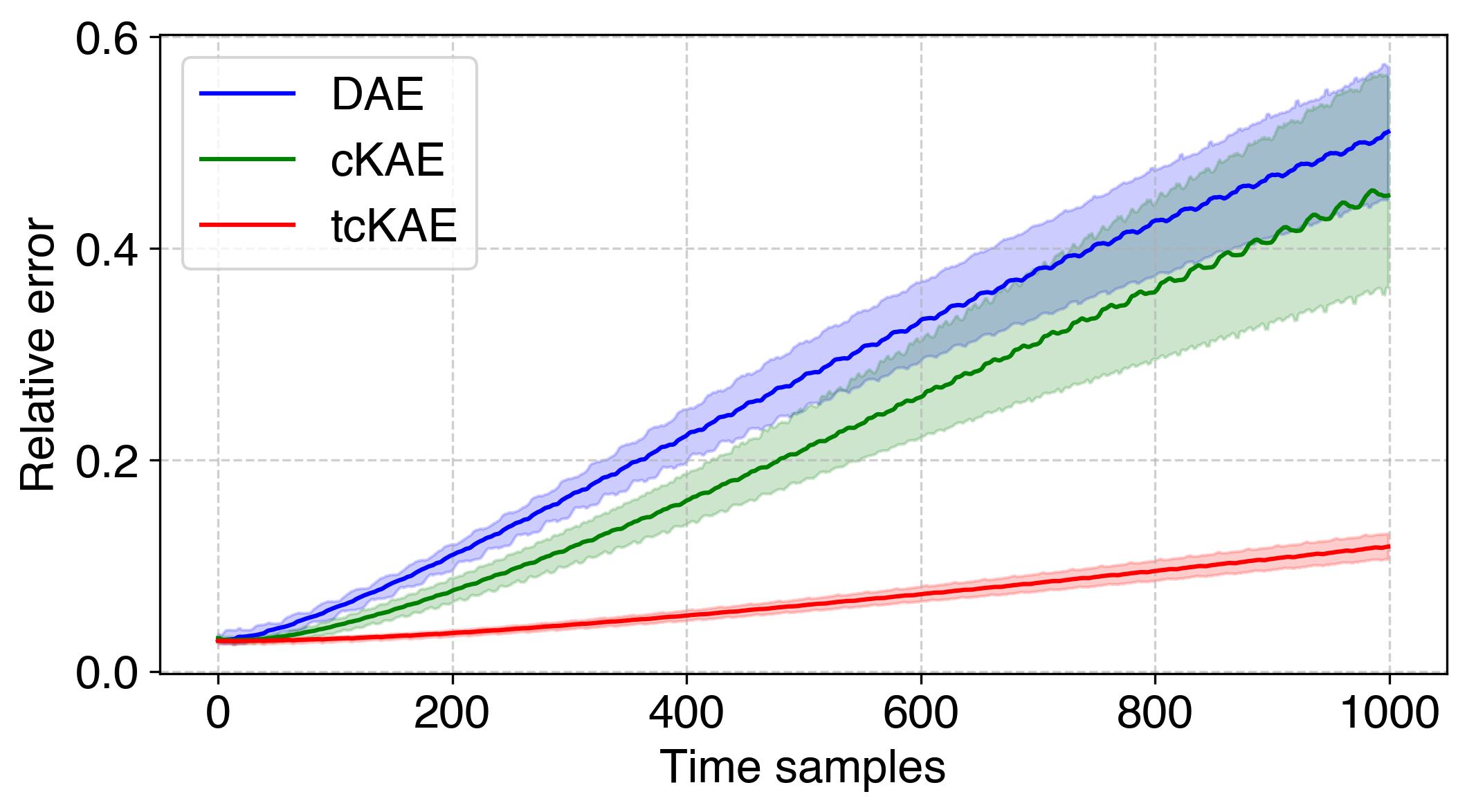}}
    \\ 
   \subfloat[$N_\text{train}=52$, Clean.\label{fig:err_pend_ntr52snr0} ]{%
        \includegraphics[width=0.49\linewidth]{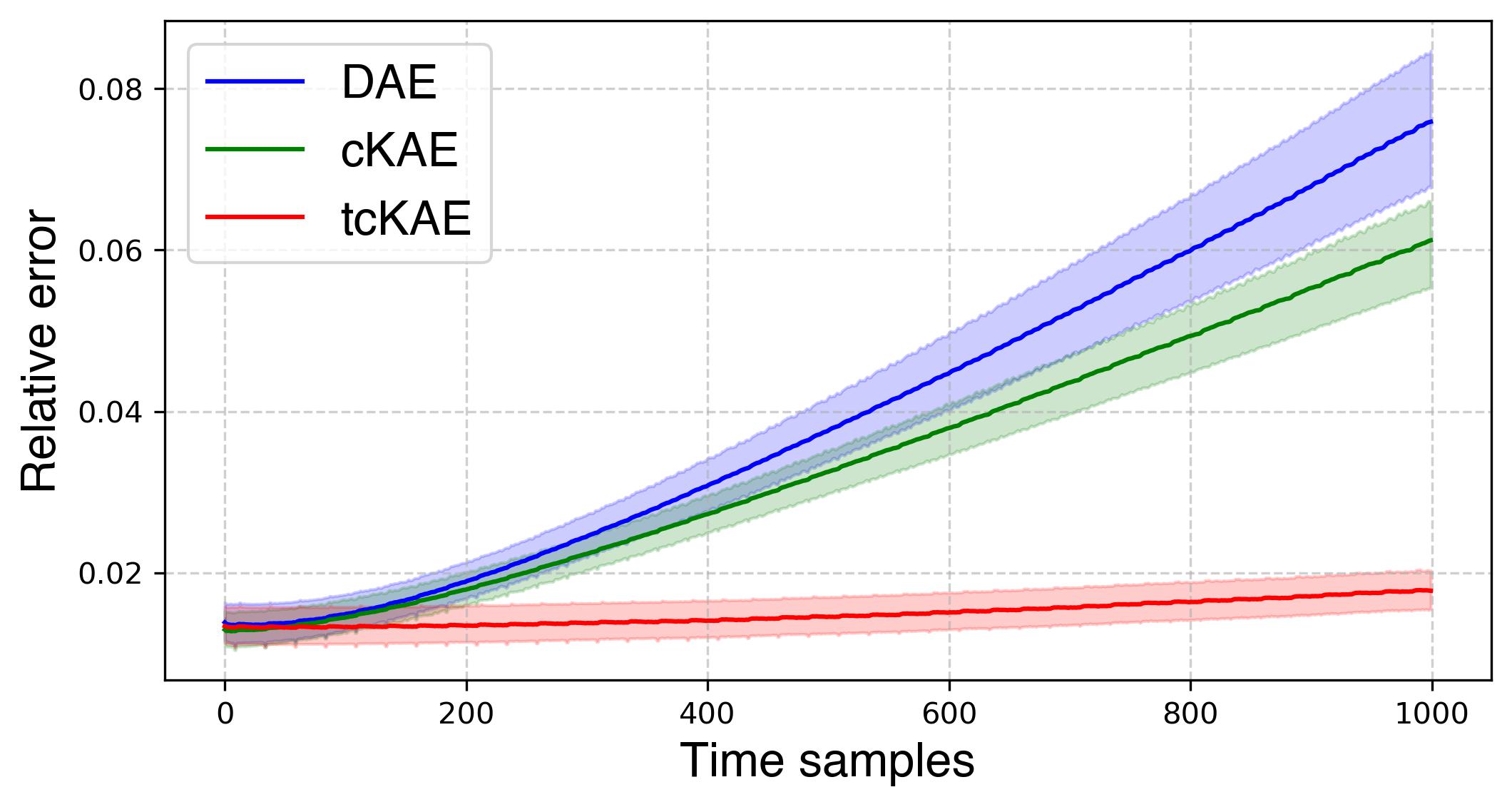}}
    \hfill 
    \subfloat[$N_\text{train}=52$, 30 dB SNR.\label{fig:err_wavy_ntr52snr30} ]{%
        \includegraphics[width=0.49\linewidth]{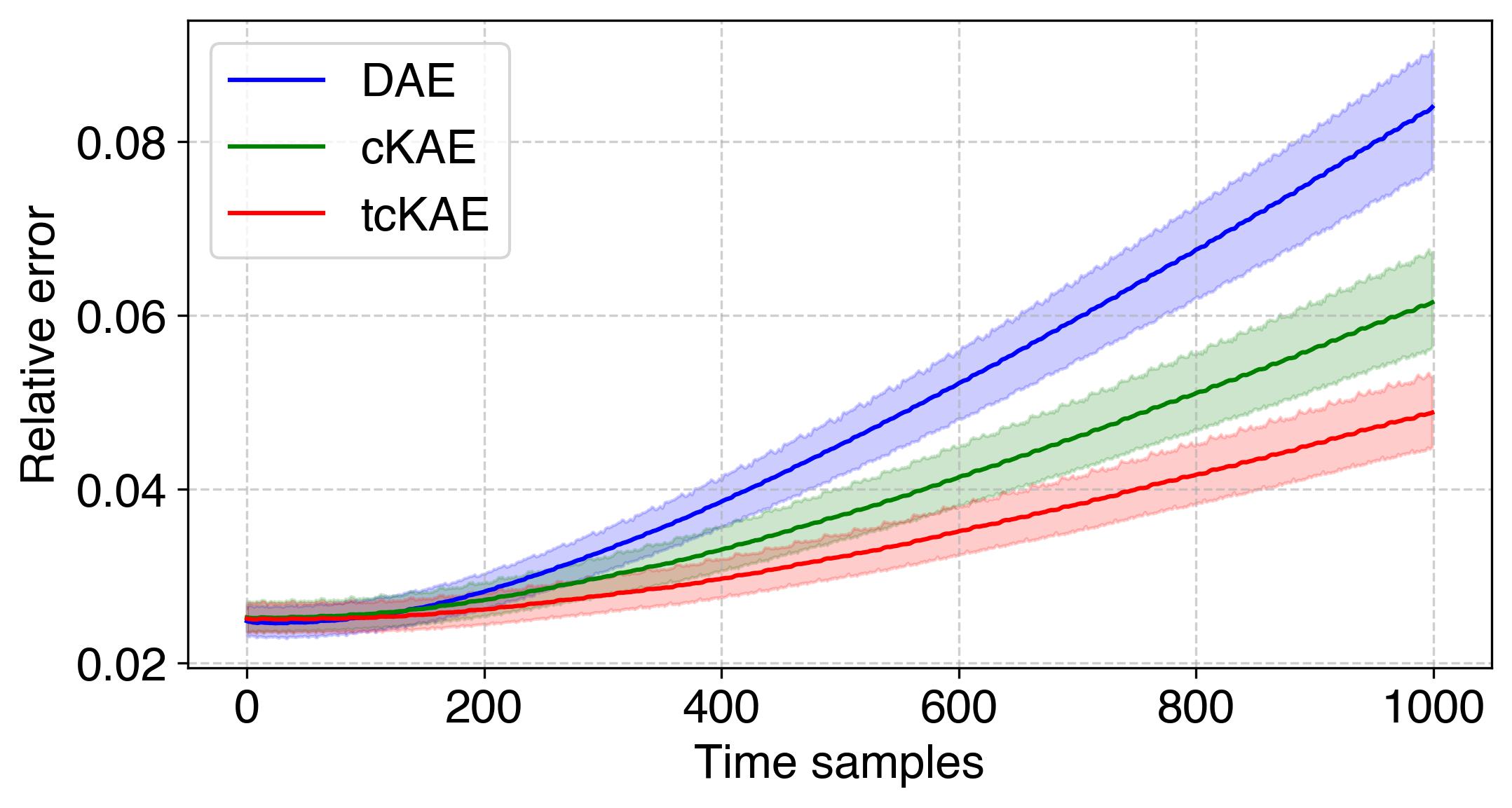}} 
  \caption{\small{Comparison between DAE, cKAE, and tcKAE for long-term extrapolation performance of oscillating electron beam. The shaded portion denotes the 90\% confidence interval.}  }\label{fig:err_electron_beam}
\end{figure*}

\begin{figure} [htbp]
    \centering
      \includegraphics[width= 1 \textwidth]{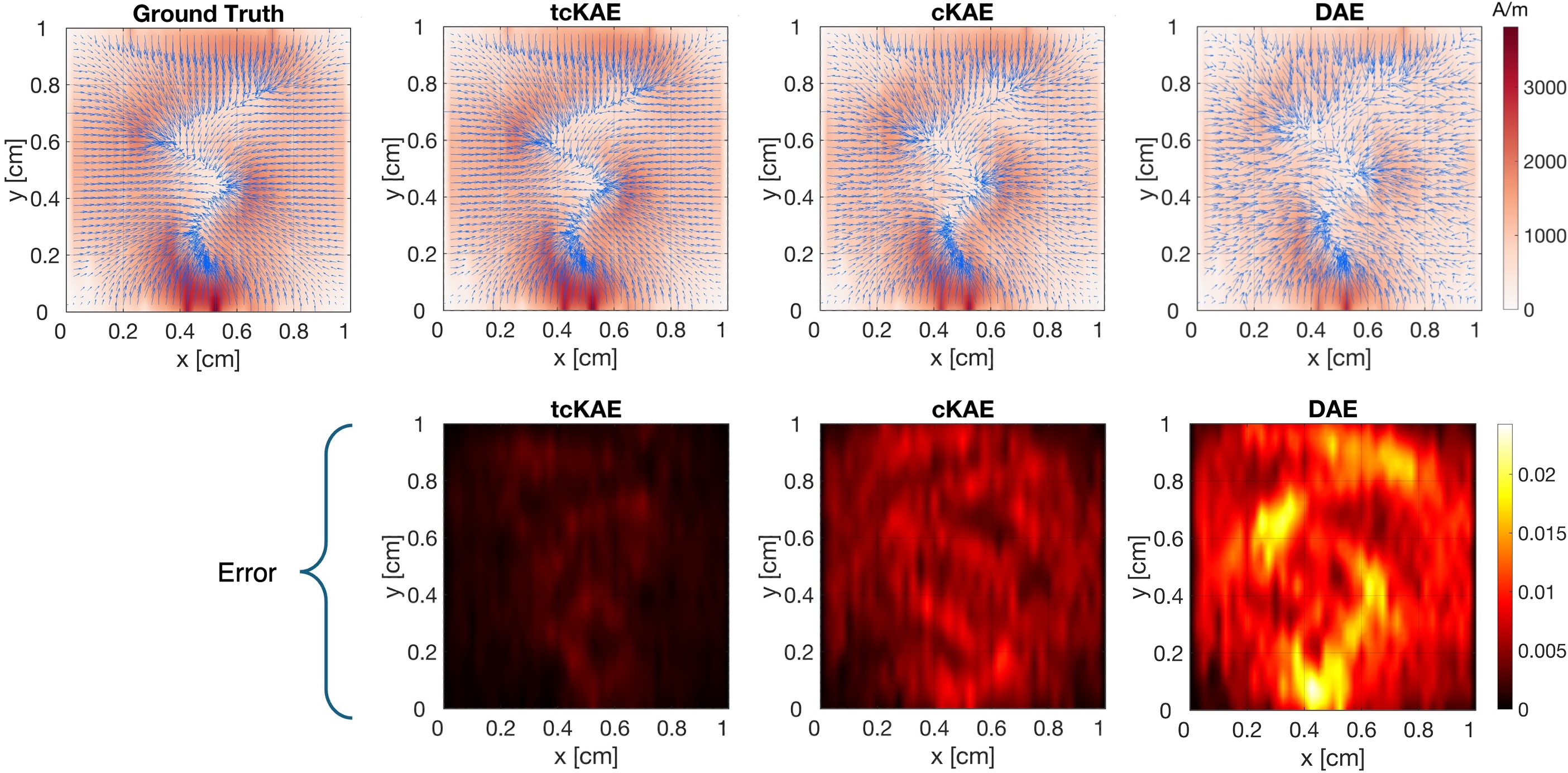}
  \caption{Predictions (top) and corresponding error profile (bottom) for oscillating electron beam example ($N_\text{train}=28$, SNR = 30 dB) at $1000^{\text{th}}$ time sample in the test-region.
   \label{fig:final_snap_error_beam}}
\end{figure}

\subsubsection*{Undamped Pendulum} The undamped pendulum is a classic textbook example of nonlinear dynamical system for benchmarking performance of data-driven predictive models \citep{azencot2020forecasting,greydanus2019hamiltonian,Chen2020Symplectic,bounou2021online,azari2022equivariant}. Although, the pendulum dynamics do not fall under typical PDE-governed high-dimensional systems, it serves as an excellent illustrative example due to the wide range of physical systems they represent, ranging from orbital mechanics, micorwave cavities to select biological and quantum systems. We use the elliptical functions \citep{azencot2020forecasting,cKAEcode} for exact solution of the motion governed by the ODE $\frac{d^2\theta}{dt^2}+\frac{g}{l}\text{sin}\theta = 0$, where $\theta\in[0,2\pi]$ denotes the angular displacement from equilibrium $\theta=0$ in radians. The gravitational constant and length of the pendulum is denoted by $g$ and $l$ respectively. We consider $g=9.8$ m/s$^2$, and $l=1$ m with initial condition $\theta=\theta_0$ and $\frac{d\theta}{dt}=\dot{\theta}=0$. We sample the solution at interval of $0.1$ s upto $220$ s, resulting in the dataset $\Theta\in \mathbb{R}^{2\times 2200}$ with $2200$ time samples. Note that the dimensionality of the state-space is two, since we have two states $\theta$ and $\dot{\theta}$. In order to mimic a high-dimensional system, we use a random orthogonal transformation \citep{azencot2020forecasting} to rotate it to higher dimensional $(N_{\text{dim}}=64)$ space, i.e. $\mathbf{X}=\mathbf{P}\cdot\Theta$, where $\mathbf{P}\in\mathbb{R}^{64\times 2}$ is a random orthogonal matrix. The final dataset $\mathbf{X}\in \mathbb{R}^{64\times 2200}$ consists of $2200$ time samples of a state with dimension $64$. We generate the data using the initial condition $\theta_0=2.4$. One cycle of oscillation is covered approximately by $31$ time steps (same as time samples). We perform the tests for $N_\text{train}=32$ and $64$, for clean and noisy ($30$ dB SNR) data. For $N_{\mathrm{train}}=32$, we take $N_{\text{val}_1}=N_{\text{val}_2}=8$, and for $N_{\mathrm{train}}=64$, $N_{\text{val}_1}=N_{\text{val}_2}=16$ is taken. We extrapolate 1000 future time samples using the trained model to evaluate the test error. We use $N_s=N_\text{train}$ initial conditions in test set to generate the 90\% confidence interval. The results are summarized in Table \ref{tab:pend_01}.

\begin{table}[htbp]
\footnotesize
\centering
\caption{Error (\%) comparison for different $N_\text{train}$ and noise levels for pendulum data. The quantities inside parenthesis show the width of 90\% confidence interval.}
\label{tab:pend_01}
\begin{tabular}{lcc|cc}
\toprule
& \multicolumn{2}{c|}{$N_\text{train}=32$} & \multicolumn{2}{c}{$N_\text{train}=64$} \\
\cmidrule(lr){2-3} \cmidrule(lr){4-5}
Method & Clean & 30 dB SNR & Clean & 30 dB SNR \\
\midrule
DAE & 34.411~(20.12) & 35.055~(18.83)  & 13.884~(12.10) & 12.421~(9.75) \\
cKAE & 26.415~(18.52) & 23.202~(15.71)  & 12.388~(11.82) & 11.316~(9.75) \\
tcKAE & 12.561~(10.20) & 15.147~(10.92)  & 7.491~(6.36) & 9.826~(7.91) \\
\bottomrule
\end{tabular}
\end{table}

As can be seen in Table \ref{tab:pend_01}, and Figure \ref{fig:err_pend}, our proposed tcKAE has clear advantage over DAE or cKAE for clean data. This advantage is magnified for limited dataset of $N_\text{train}=32$, which covers approximately one and half cycle of oscillation of the pendulum. The advantage tends to diminish with increasing size of training set. Large $N_\text{train}$ not only makes learning with larger look-ahead step $k_m$ possible, but also reduces overfitting, resulting in enhanced stability for long-term predictions for DAE and cKAE. Table \ref{tab:pend_01} also shows that the variation of error (values inside parenthesis) is lowest for tcKAE. We see a similar trend when noise is included. However, interestingly with addition of noise in some cases the relative error actually decreases. This should not be surprising since models tend to overfit with limited data, and adding noise is one of the solutions to avoid such overfitting. Overall, tcKAE provides more accurate estimate of the state for long-term prediction.

\subsubsection*{Oscillating electron beam} Plasma systems are excellent candidates for data-driven modeling due to their high-diemnsionality, and inherent nonlinearity arising from the complex wave-particle interaction. A two dimensional electron beam (Figure \ref{fig:beam_snap_n40000}) oscillating under the influence of an external transverse magnetic flux is simulated using charge-conserving electromagnetic particle-in-cell (EMPIC) algorithm~\citep{na2016local}. The solution domain ($1$ m $\times$ $1$ m) is discretized using unstructured triangular mesh with 844 nodes, 2447 edges, and 1604 elements (triangles). The electric field data $(\mathbf{e}\in\mb{R}^{2447})$ is sampled at every 0.32 ns. Starting from $80$ ns, total $1751$ time samples are collected with the complete dataset $\mf{X}\in\mb{R}^{2447\times 1751}$. More details are provided in Methods: Simulation Setup section .

For the electron beam oscillation, one period approximately comprises of 25 time samples. In order to assess the performance of the Koopman models for limited dataset, we train the models for $N_\text{train}=28$ and $52$, each for clean and noisy (30 dB SNR) data. For $N_\text{train}=28$; $N_{\text{val}_1}=N_{\text{val}_2}=7$, and for $N_\text{train}=52$; $N_{\text{val}_1}=N_{\text{val}_2}=13$. We extrapolate up to 1000 time samples for evaluating the test error. Growth of the relative prediction error over 1000 time samples is depicted in Figure~\ref{fig:err_electron_beam}. The snapshots of the electric field distribution and corresponding spatial error profiles for long-term prediction is shown in Figure~\ref{fig:final_snap_error_beam}. Note that the spatial error is calculated as following: $\delta(x,y)=\frac{|\mathbf{x}(x,y)-\hat{\mathbf{x}}(x,y)|}{||\mathbf{x}||_2}$. As can be seen from Figure~\ref{fig:err_electron_beam} and Table \ref{tab:osc}, the advantage of tcKAE is prominent for limited dataset i.e. $N_\text{train}=28$.


\begin{table}[htbp]
\footnotesize
\centering
\caption{Error (\%) comparison for different $N_\text{train}$ and noise levels for oscillating electron beam data. The quantities inside parenthesis show
the width of 90\% confidence interval.}
\label{tab:osc}
\begin{tabular}{lcc|cc}
\toprule
& \multicolumn{2}{c|}{$N_\text{train}=28$} & \multicolumn{2}{c}{$N_\text{train}=52$} \\
\cmidrule(lr){2-3} \cmidrule(lr){4-5}
Method & Clean & 30 dB SNR & Clean & 30 dB SNR \\
\midrule
DAE & 16.570~(4.22) & 27.006~(6.20)  & 3.956~(0.87) & 4.782~(0.74) \\
cKAE & 14.958~(2.15) & 21.901~(8.28)  & 3.370~(0.61) & 3.900~(0.64) \\
tcKAE & 4.003~(0.73) & 6.575~(1.23)  & 1.489~(0.46) & 3.376~(0.52) \\
\bottomrule
\end{tabular}
\end{table}

\begin{figure*} [t]

    \centering
  \subfloat[$N_\text{train}=28$, Clean. \label{fig:err_flow_ntr28snr0} ]{%
       \includegraphics[width=0.49\linewidth]{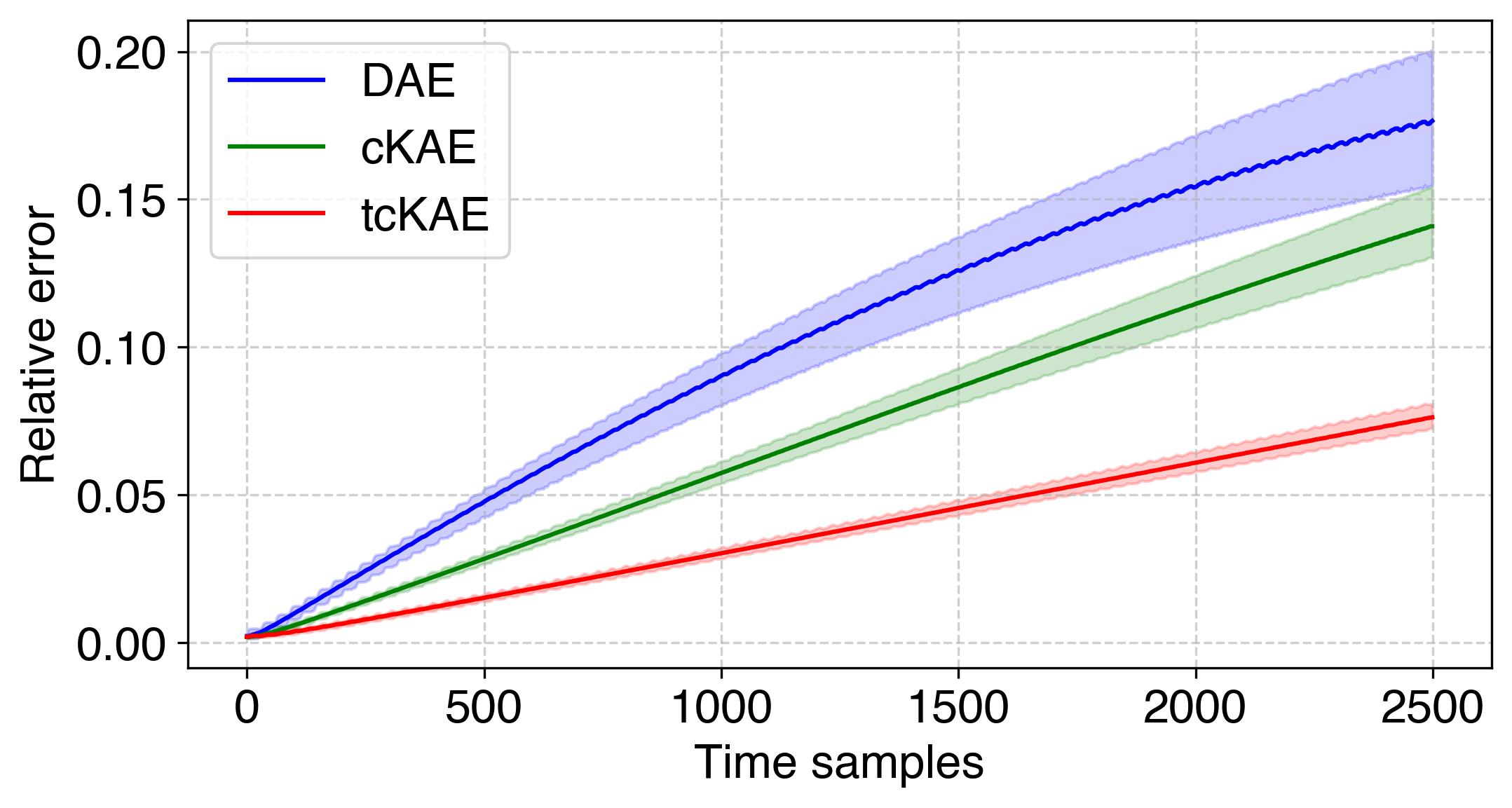}}
    \hfill
  \subfloat[ $N_\text{train}=28$, 30 dB SNR. \label{fig:err_flow_ntr28snr30} ]{%
        \includegraphics[width=0.49\linewidth]{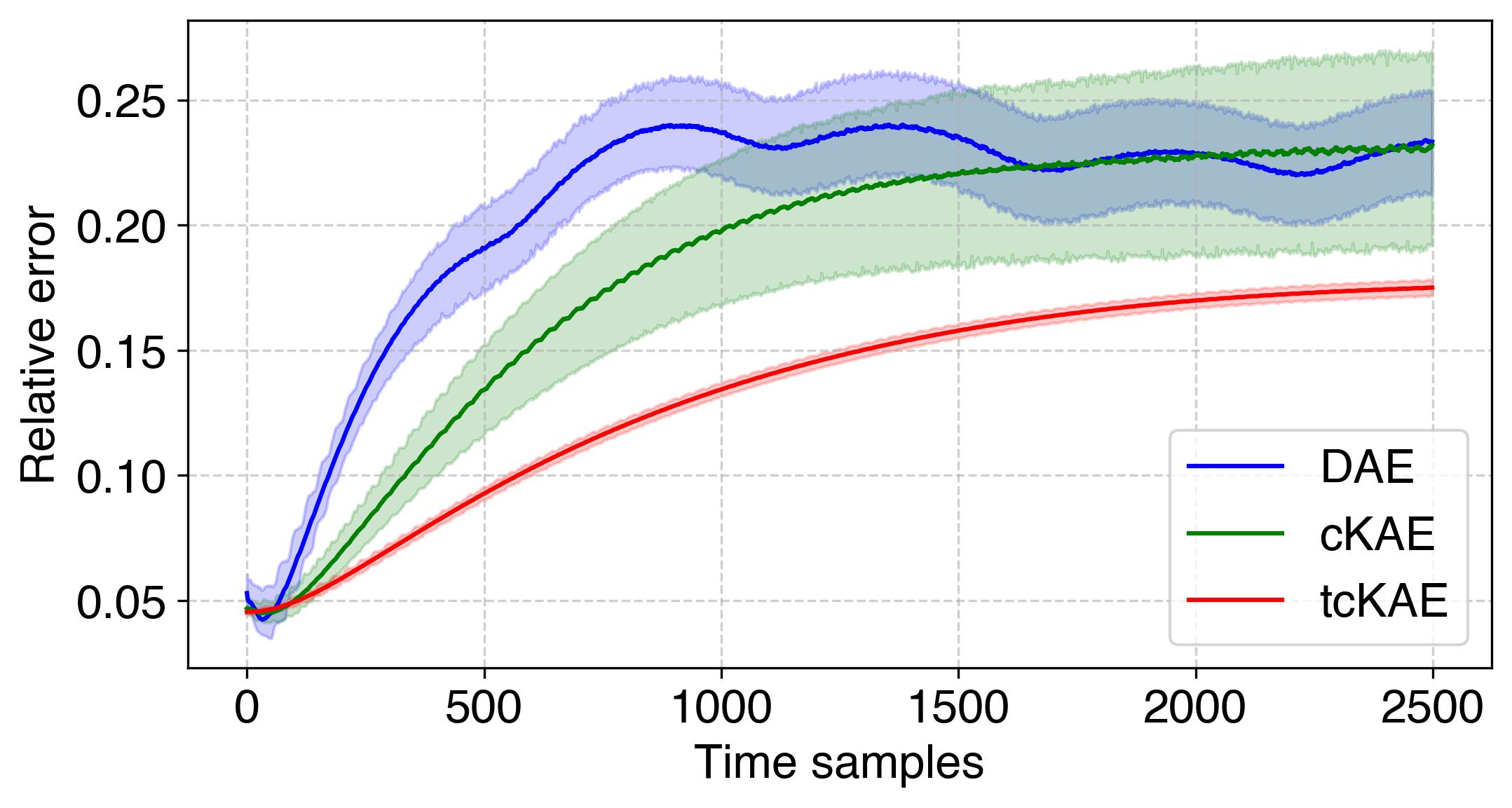}}
    \\ 
   \subfloat[$N_\text{train}=40$, Clean.\label{fig:err_pend_ntr40snr0} ]{%
        \includegraphics[width=0.49\linewidth]{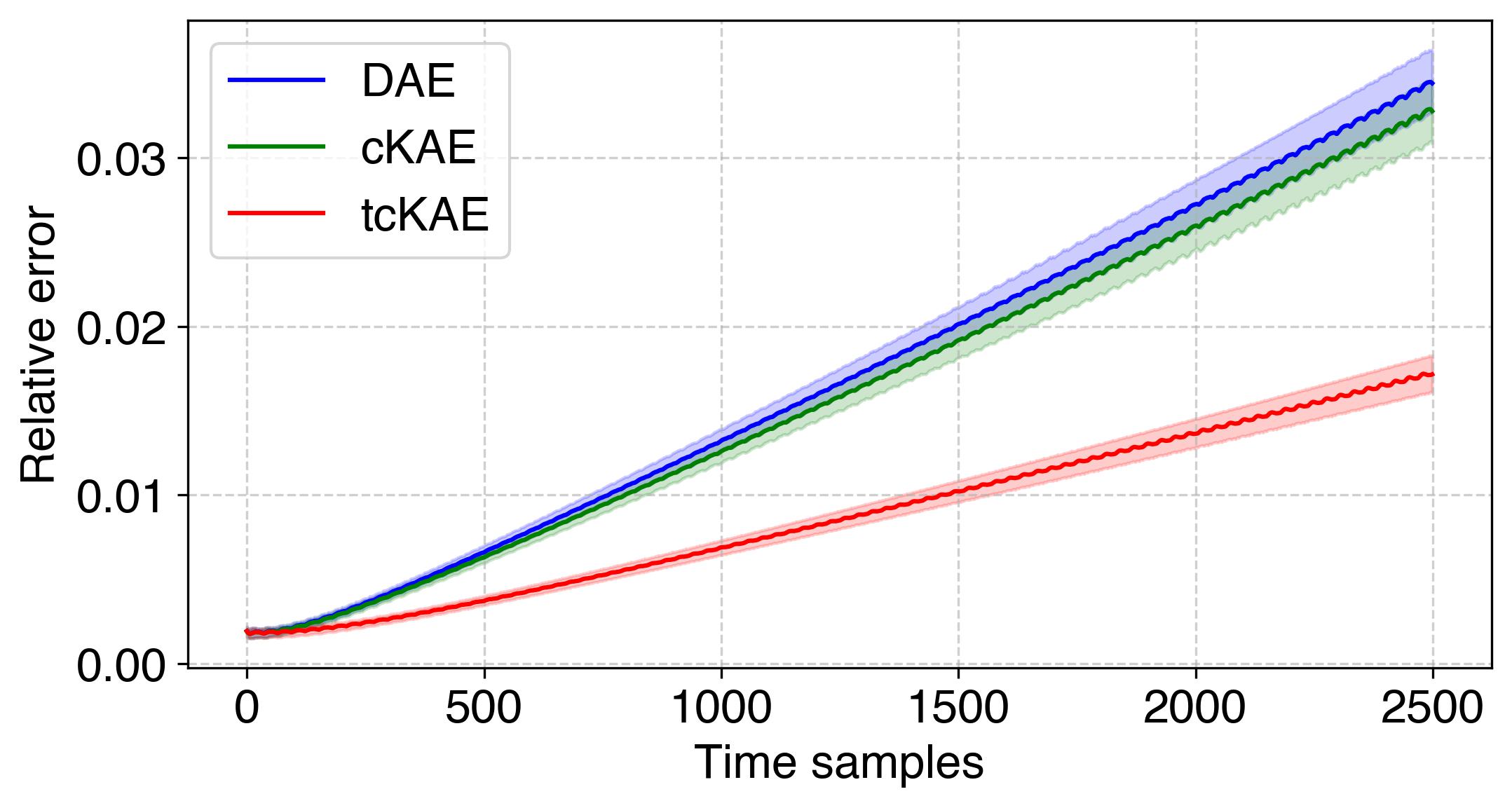}}
    \hfill 
    \subfloat[$N_\text{train}=40$, 30 dB SNR.\label{fig:err_flow_ntr40snr30} ]{%
        \includegraphics[width=0.49\linewidth]{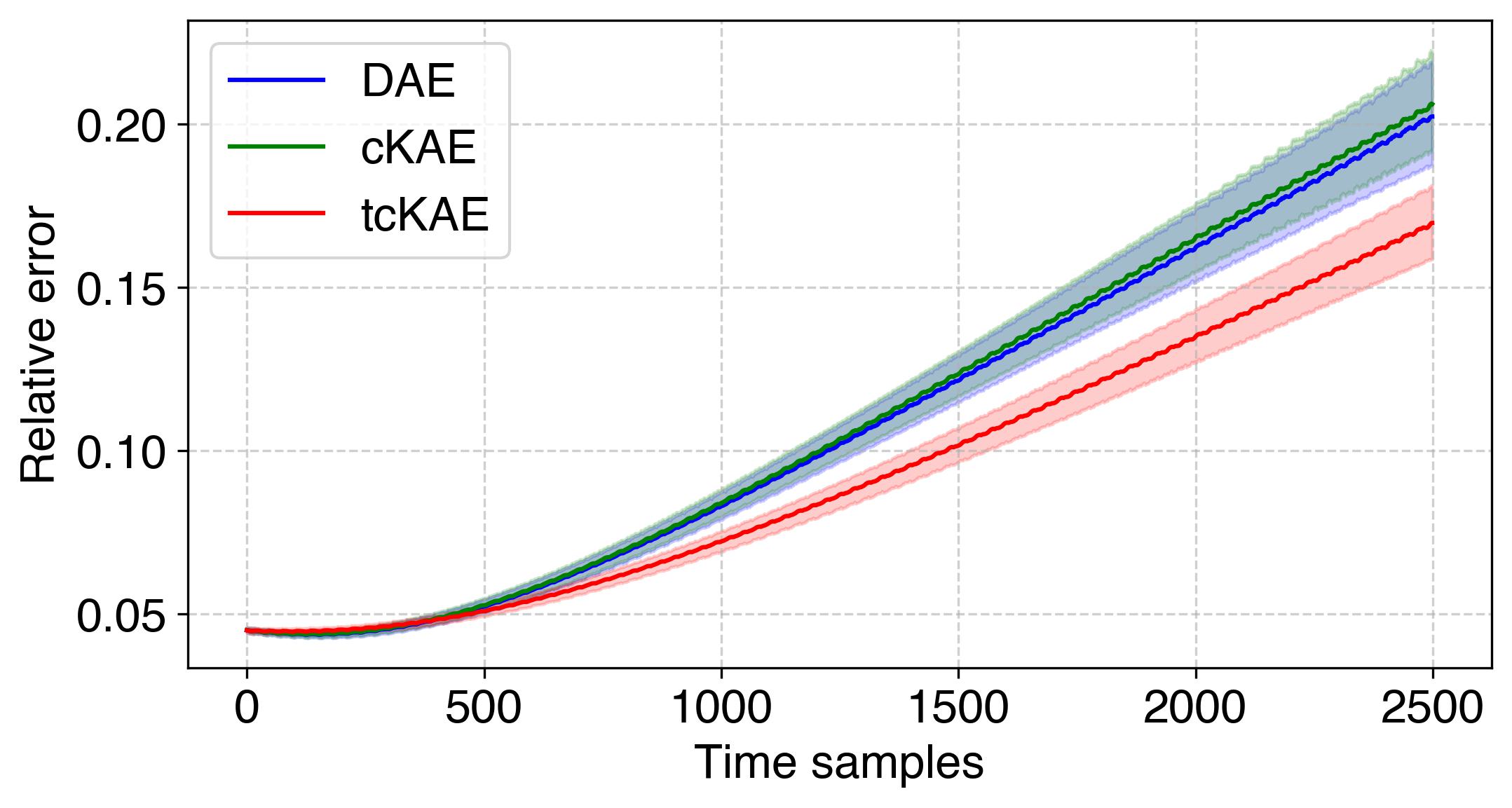}} 
  \caption{\small{Comparison between DAE, cKAE, and tcKAE for long-term extrapolation performance of flow past cylinder. The shaded portion denotes the 90\% confidence interval.}  }\label{fig:err_FPC}
\end{figure*}

\subsubsection*{Flow past cylinder}
Flow past cylinder is another common dynamical system used for benchmarking data-driven models. The simulation was carried out using MATLAB's FEAtool (Figure \ref{fig:flow_cyl_snap}). A cylinder with diameter 0.1 m is located at a height of 0.2 m. The fluid is characterized by its density $\rho=1$ Kg/m$^3$, and dynamic viscosity $\mu=0.001$ Kg/m s. The flow is unsteady with a maximum velocity of 1 m/s and mean velocity $\frac{2}{3}$ of the maximum velocity. The simulation is run for 80 s until the steady state is achieved, and the horizontal $u$ component of the velocity is probed at every 0.02 s starting from 20 s. The complete dataset $\mf{X}\in\mb{R}^{2647\times 3001}$ consists of 3001 time samples. More details regarding the simulation setup is presented in Methods: Simulation Setup section. Growth of relative error over 2500 time samples is shown in Figure~\ref{fig:err_FPC}. Predicted velocity field snapshots and corresponding error profile for long-term prediction is shown in Figure~\ref{fig:final_snap_error_FPC}. Here also we observe that tcKAE provides significant advantage for limited training data.

\begin{table}[htbp]
\footnotesize
\centering
\caption{Error (\%) comparison for different $N_\text{train}$ and noise levels for flow past cylinder data. The quantities inside parenthesis show
the width of 90\% confidence interval.}
\label{tab:flow}
\begin{tabular}{lcc|cc}
\toprule
& \multicolumn{2}{c|}{$N_\text{train}=28$} & \multicolumn{2}{c}{$N_\text{train}=40$} \\
\cmidrule(lr){2-3} \cmidrule(lr){4-5}
Method & Clean & 30 dB SNR & Clean & 30 dB SNR \\
\midrule
DAE & 10.182~(2.24) & 20.724~(3.60)  & 1.695~(0.18) & 10.775~(1.29) \\
cKAE & 7.153~(1.06) & 18.376~(5.58)  & 1.616~(0.17) & 10.933~(1.24) \\
tcKAE & 3.808~(0.45) & 13.311~(0.53)  & 0.874~(0.11) & 9.275~(0.94) \\
\bottomrule
\end{tabular}
\end{table}

In case of flow past cylinder, since one cycle approximately contains 29 time samples, we consider  $N_\text{train}=28$ and $40$. For $N_{\mathrm{train}}=28$, we take $N_{\text{val}_1}=N_{\text{val}_2}=7$, and for $N_{\mathrm{train}}=40$, $N_{\text{val}_1}=N_{\text{val}_2}=10$ is taken. We evaluate the test error for 2500 time samples in the extrapolation region. Growth of relative error over 2500 time samples is shown in Figure~\ref{fig:err_FPC}. Predicted velocity field snapshots and corresponding error profile for long-term prediction is shown in Figure~\ref{fig:final_snap_error_FPC}. Table~\ref{tab:flow} provides the relative (percentage) time-averaged test error. Here also we observe that tcKAE provides significant advantage for limited training data.





 


\begin{figure} [h]
    \centering
      \includegraphics[width= 0.75 \textwidth]{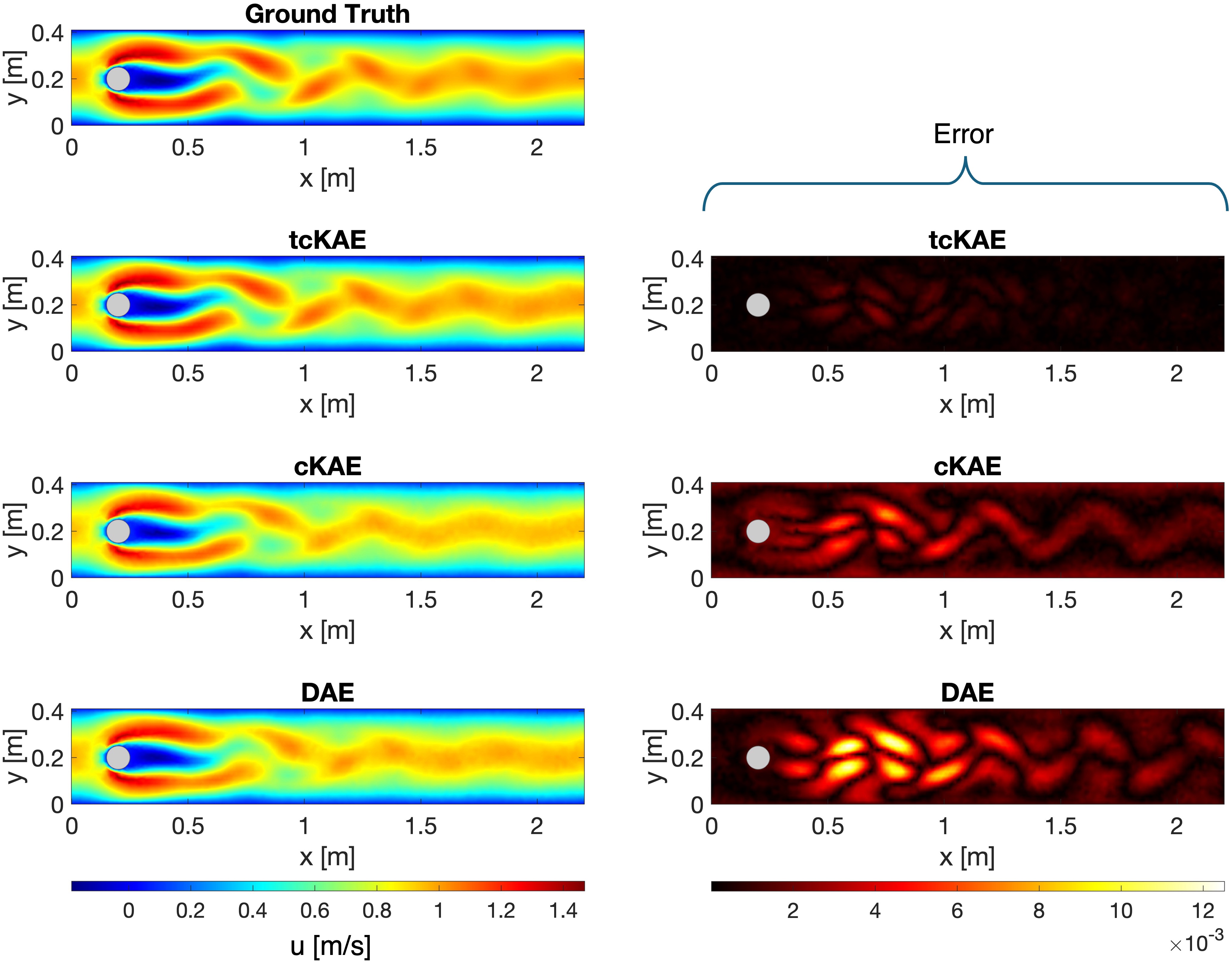}
  \caption{Predictions (left) and corresponding error profile (right) for flow past cylinder example ($N_\text{train}=28$, clean) at $2500^{\text{th}}$ time sample in the test-region. The grey circle represents the cylinder. 
   \label{fig:final_snap_error_FPC}}
\end{figure}

\section*{Discussion}

The proposed tcKAE algorithm advances the state-of-the-art by enforcing multi-step temporal consistency in the latent space. Unlike prior KAE architectures such as consistent Koopman autoencoders (cKAE) \citep{azencot2020forecasting}, which rely on backward dynamics assumptions, tcKAE sidesteps this limitation by introducing a consistency regularization term in the forward direction. The ablation study in the supplementary material further demonstrates that, even without the backward and consistency losses, the combination of forward, identity, and temporal consistency losses achieves comparable or even superior accuracy to cKAE. This innovation strengthens the Koopman invariance of the latent space, enhancing robustness to limited date and noise, improving generalizability. Furthermore, empirical results demonstrate that tcKAE outperforms both cKAE and dynamic autoencoders (DAEs) \citep{lusch2018deep} in terms of long-term prediction accuracy and stability, particularly for scenarios where we have limited training data. The proposed tcKAE not only achieved lower error margins in benchmark problems such as pendulum oscillation, oscillating electron beam, and flow past a cylinder, but also demonstrated narrower confidence intervals, thereby inspiring greater confidence in the predictions. It is important to note that the proposed tcKAE does not guarantee unconditional stability. The tcKAE predictions are \textit{more} stable for long-term predictions compared to DAE and cKAE. However, the secular growth in relative error indicates that if we were to predict 10 or 20 times further, the large error would likely render the results useless for all the models.

Despite its advantages, tcKAE has some limitations. With sufficient training data, the performance gap between tcKAE and vanilla KAEs narrows significantly, reducing the need for additional consistency constraints. Furthermore, the inclusion of temporal regularization introduces computational overhead, increasing training time compared to simpler DAE method. These trade-offs highlight the necessity of tcKAE primarily in scenarios with scarce or noisy data, where its robustness and enhanced prediction accuracy provide a critical edge. This makes tcKAE particularly attractive for modeling high-dimensional nonlinear systems, where computationally costly numerical solvers limit the amount of training data available. Note that although, the proposed idea of temporal consistency is explored in the context of KAEs, it can be easily extended to other models. Future work may also explore control applications with tcKAE and hybrid models that balance the computational cost of tcKAE with the efficiency of traditional approaches, potentially by dynamically adapting regularization strength based on data quality and training progress.


\section*{Data Availability}
The code for generating the data and results of this paper will be made available upon request. Any such request should be addressed to A.C. It is also partially available at this dynamically updated GitHub \href{https://github.com/SOARLabOSU/tcKAE}{repository}.

\bibliography{main}



\section*{Acknowledgements}

This work was funded in part by the Department of Energy Grant DE-SC0022982 through the NSF/DOE Partnership in Basic Plasma Science and Engineering and in part by the Ohio Supercomputer Center Grants PAS-0061 and PAS-2709.

\section*{Author contributions statement}

I.N. and D.G. conceived the algorithm.  I.N. developed the code. A.C. conducted the experiments. I.N. and A.C. analyzed the results. F.T., M.K., and D.G. advised I.N and A.C. and acquired the funding and resources. All authors reviewed the manuscript. 

\section*{Additional information}


\textbf{Competing interests}

\noindent
The authors declare no competing interests.





\end{document}


\flushbottom
\newtitle

\section*{Training details}
The crucial hyperparameters to tune are learning rate $(l_r)$, learning rate decay rate $(l_{rd})$, corresponding decay schedule, maximum look-ahead step $k_m$, the weigths of the individual component of the total loss, i.e. $\gamma_\text{id}, \gamma_\text{fwd}, \gamma_\text{bwd}, \gamma_\text{con}$ and \textcolor{red}{$\gamma_\text{tc}$}. Also, for each case we keep $\gamma_\text{id}=1$ constant. Among structural hyperparameters, $N_h$, $N_l$ are crucial. Note that as mentioned in the previous subsection, $N_h$ is varied by varying $\alpha$. We provide values of these hyperparameters for extreme scenarios for each of the test-case. Note that $N_\text{in}=N_\text{out}$ are not tunable, and depend on the dimension of the input data. The pendulum cases is trained for 600 epochs whereas electron beam and fluid flow cases are trained for 1000 epochs.

\begin{table}[htbp]
\centering
\caption{Training hyperparameters of DAE, cKAE, and tcKAE for pendulum}
\label{tab:pend}
\begin{tabular}{lccccccccccccc}
\toprule
 & Method & $N_\text{in}=N_\text{out}$ & $N_h$ & $N_l$ & $l_r$ & $l_{rd}$ & $k_m$ & $\gamma_\text{fwd}$ & $\gamma_\text{bwd}$ & $\gamma_\text{con}$ & $\gamma_{tc}$ & $k_{tm}$ & $e_s$  \\
\midrule
\multirow{3}{*}{\makecell{$N_\text{train}=32$,\\ clean}} & DAE & 64 & 64 & 12 & 1e-2  & 0.5 & 16 & 4 & - & - & - & - & - \\
 & cKAE & 64 & 96 & 16 & 1e-2  & 0.5 & 16 & 4 & 2 & 1e-4 & - & - & -  \\
 & tcKAE & 64 & 64 & 16 & 1e-2  & 0.5 & 16 & 4 & 2 & 1e-4 & 4 & 8 & 50  \\
\midrule
\multirow{3}{*}{\makecell{$N_\text{train}=32$,\\ 30 dB noise}} & DAE & 64 & 64 & 10 & 1e-2  & 0.5 & 16 & 6 & - & - & - & - & - \\
 & cKAE & 64 & 96 & 12 & 1e-2  & 0.5 & 16 & 4 & 6 & 1e-5 & - & - & -  \\
 & tcKAE & 64 & 96 & 16 & 1e-2  & 0.5 & 16 & 4 & 6 & 1e-5 & 1 & 20 & 50  \\
\midrule
\multirow{3}{*}{\makecell{$N_\text{train}=64$,\\  clean}} & DAE & 64 & 16 & 8 & 1e-2  & 0.5 & 16 & 1 & - & - & - & - & - \\
 & cKAE & 64 & 96 & 20 & 1e-2  & 0.5 & 16 & 0.1 & 0.1 & 1e-7 & - & - & -  \\
 & tcKAE & 64 & 16 & 8 & 1e-2  & 0.5 & 16 & 4 & 4 & 1e-4 & 1 & 10 & 20  \\
 \midrule
 \multirow{3}{*}{\makecell{$N_\text{train}=64$,\\ $30$ dB  noise}} & DAE & 64 & 64 & 16 & 1e-2  & 0.5 & 16 & 0.1 & - & - & - & - & - \\
 & cKAE & 64 & 96 & 20 & 1e-2  & 0.5 & 16 & 0.1 & 0.1 & 1e-5 & - & - & -  \\
 & tcKAE & 64 & 16 & 8 & 1e-2  & 0.5 & 16 & 2 & 1 & 1e-5 & 2 & 10 & 50  \\
\bottomrule
\end{tabular}
\end{table}

\subsection*{Undamped pendulum}
 Please see Table \ref{tab:pend}. We schedule the decay in learning rate at following epochs: $[30,100,200,400]$.

\subsection*{Oscillating electron beam}
Please see Table \ref{tab:osc_beam}. The learning rate decay schedule for oscillating electron beam is $[30,~200,~400,~700]$, i.e. learning rate is reduced by a factor of $0.5$ at epochs 30,70, 200 and 700.
\begin{table}[htbp]
\centering
\caption{Training hyperparameters of DAE, cKAE, and tcKAE for oscillating electron beam}
\label{tab:osc_beam}
\begin{tabular}{lccccccccccccc}
\toprule
 & Method & $N_\text{in}=N_\text{out}$ & $N_h$ & $N_l$ & $l_r$ & $l_{rd}$ & $k_m$ & $\gamma_\text{fwd}$ & $\gamma_\text{bwd}$ & $\gamma_\text{con}$ & $\gamma_{tc}$ & $k_{tm}$ & $e_s$  \\
\midrule
\multirow{3}{*}{\makecell{$N_\text{train}=28$,\\ clean}} & DAE & 2447 & 320 & 16 & 1e-3  & 0.5 & 20 & 0.5 & - & - & - & - & - \\
 & cKAE & 2447 & 320 & 16 & 1e-3  & 0.5 & 20 & 2 & 0.5 & 1e-7 & - & - & -  \\
 & tcKAE & 2447 & 192 & 36 & 1e-3  & 0.5 & 20 & 1 & 0.5 & 1e-7 & 1e-2 & 12 & 50  \\
\midrule
\multirow{3}{*}{\makecell{$N_\text{train}=28$,\\ 30 dB noise}} & DAE & 2447 & 320 & 16 & 1e-3  & 0.5 & 20 & 0.5 & - & - & - & - & - \\
 & cKAE & 2447 & 192 & 20 & 1e-3  & 0.5 & 20 & 0.5 & 2 & 1e-4 & - & - & -  \\
 & tcKAE & 2447 & 128 & 48 & 1e-3  & 0.5 & 20 & 0.5 & 2 & 1e-4 & 1e-2 & 10 & 100  \\
\midrule
\multirow{3}{*}{\makecell{$N_\text{train}=52$,\\  clean}} & DAE & 2447 & 192 & 36 & 1e-3  & 0.5 & 20 & 2 & - & - & - & - & - \\
 & cKAE & 2447 & 320 & 36 & 1e-3  & 0.5 & 20 & 2 & 1e-3 & 1e-5 & - & - & -  \\
 & tcKAE & 2447 & 192 & 36 & 1e-3  & 0.5 & 20 & 2 & 1e-3 & 1e-5 & 1e-2 & 24 & 100  \\
 \midrule
 \multirow{3}{*}{\makecell{$N_\text{train}=52$,\\ $30$ dB  noise}} & DAE & 2447 & 320 & 36 & 1e-3  & 0.5 & 20 & 2 & - & - & - & - & - \\
 & cKAE & 2447 & 256 & 24 & 1e-3  & 0.5 & 20 & 2 & 1e-2 & 1e-6 & - & - & -  \\
 & tcKAE & 2447 & 256 & 24 & 1e-3  & 0.5 & 20 & 2 & 1e-2 & 1e-6 & 1e-4 & 28 & 50  \\
\bottomrule
\end{tabular}
\end{table}

\subsection*{Flow past cylinder}
Please see Table \ref{tab:flow2}. The learning rate decay schedule (epochs) is $[30,~200,~400,~700]$.
\begin{table}[htbp]
\centering
\caption{Training hyperparameters of DAE, cKAE, and tcKAE for flow past cylinder}
\label{tab:flow2}
\begin{tabular}{lccccccccccccc}
\toprule
 & Method & $N_\text{in}=N_\text{out}$ & $N_h$ & $N_l$ & $l_r$ & $l_{rd}$ & $k_m$ & $\gamma_\text{fwd}$ & $\gamma_\text{bwd}$ & $\gamma_\text{con}$ & $\gamma_{tc}$ & $k_{tm}$ & $e_s$  \\
\midrule
\multirow{3}{*}{\makecell{$N_\text{train}=28$,\\ clean}} & DAE & 2647 & 192 & 96 & 1e-3  & 0.5 & 15 & 2 & - & - & - & - & - \\
 & cKAE & 2647 & 256 & 92 & 1e-3  & 0.5 & 15 & 2 & 2 & 1e-7 & - & - & -  \\
 & tcKAE & 2647 & 320 & 112 & 1e-3  & 0.5 & 15 & 2 & 4 & 1e-7 & 1e-1 & 10 & 100  \\
\midrule
\multirow{3}{*}{\makecell{$N_\text{train}=28$,\\ 30 dB noise}} & DAE & 2647 & 256 & 80 & 1e-3  & 0.5 & 15 & 1e-2 & - & - & - & - & - \\
 & cKAE & 2647 & 192 & 48 & 1e-3  & 0.5 & 15 & 1e-2 & 0.1 & 1e-5 & - & - & -  \\
 & tcKAE & 2647 & 256 & 80 & 1e-3  & 0.5 & 15 & 0.1 & 0.1 & 1e-5 & 0.1 & 10 & 100  \\
\midrule
\multirow{3}{*}{\makecell{$N_\text{train}=40$,\\  clean}} & DAE & 2647 & 320 & 54 & 1e-3  & 0.5 & 15 & 1 & - & - & - & - & - \\
 & cKAE & 2647 & 320 & 54 & 1e-3  & 0.5 & 15 & 1 & 1e-4 & 1e-6 & - & - & -  \\
 & tcKAE & 2647 & 320 & 48 & 1e-3  & 0.5 & 15 & 1 & 1e-6 & 1e-5 & 1e-3 & 30 & 50  \\
 \midrule
 \multirow{3}{*}{\makecell{$N_\text{train}=40$,\\ $30$ dB  noise}} & DAE & 2647 & 320 & 54 & 1e-3  & 0.2 & 20 & 4 & - & - & - & - & - \\
 & cKAE & 2647 & 320 & 54 & 1e-3  & 0.2 & 20 & 4 & 1e-4 & 1e-5 & - & - & -  \\
 & tcKAE & 2647 & 320 & 54 & 1e-3  & 0.2 & 20 & 4 & 1e-4 & 1e-6 & 1e-4 & 50 & 50  \\
\bottomrule
\end{tabular}
\end{table}

\section*{Ablation study}
Note that the weights $\gamma_\text{id}$ and $\gamma_\text{fwd}$ are absolute essential. The ablation study essentially contains one extra case: one where we do not consider the loss components pertaining to backward dynamics, by chosing $\gamma_\text{bwd}=\gamma_{con}=0$. Table~\ref{tab:abla_pend_01} presents the ablation test results for the pendulum, Table~\ref{tab:abl_plasma} presents it for the oscillating electron beam data, and Table~\ref{tab:abl_flow} contains the ablation test results for the flow past cylinder data. 

\section*{Effect of $k_{tm}$}
We vary $k_{tm}$ and keep the rest of the optimal tcKAE hyperparameters fixed to assess the effect of $k_{tm}$ on the prediction accuracy. Table~\ref{tab:steps_tc_pend} presents the results for the pendulum, Table~\ref{tab:steps_tc_osc} for the oscillating electron beam, and Table~\ref{tab:steps_tc_flow} for the flow past cylinder.
\begin{table}[htbp]
\centering
\caption{Ablation study for undamped pendulum keeping $\gamma_\text{bwd}=\gamma_{con}=0$.}
\label{tab:abla_pend_01}
\begin{tabular}{lcccccc}
\toprule
 & $\gamma_\text{id}$ & $\gamma_\text{fwd}$ & $\gamma_\text{bwd}$ & $\gamma_\text{con}$ & \textcolor{blue}{$\gamma_\text{tc}$} & Average relative error$(\%)$ \\
\midrule
\multirow{4}{*}{$N_\text{train}=32$, clean} & 1 & 4 & 0 & 0 & 0 & 34.411~(20.12) \\
 & 1 & 4 & 0 & 0 & 4 & 9.225~(6.37) \\
 & 1 & 4 & 2 & 1e-4 & 0 & 26.415~(18.52) \\
 & 1 & 4 & 2 & 1e-4 & 4 & 12.561~(10.20)  \\
\midrule
\multirow{4}{*}{$N_\text{train}=32$, 30 dB noise} & 1 & 6 & 0 & 0 & 0 & 35.055~(18.83) \\
 & 1 & 6 & 0 & 0 & 4 & 12.000~(8.70) \\
 & 1 & 4 & 6 & 1e-5 & 0 & 23.202~(15.71) \\
 & 1 & 4 & 6 & 1e-5 & 1 & 15.145~(10.92) \\
\midrule
\multirow{4}{*}{$N_\text{train}=64$, clean} & 1 & 1 & 0 & 0 & 0 & 13.884~(21.10) \\
 & 1 & 1 & 0 & 0 & 2 & 7.121~(5.28) \\
 & 1 & 0.1 & 0.1 & 1e-7 & 0 & 12.388~(11.82) \\
 & 1 & 4 & 4 & 1e-4 & 1 & 7.491~(6.36) \\
\midrule
\multirow{4}{*}{$N_\text{train}=64$, 30 dB noise} & 1 & 0.1 & 0 & 0 & 0 & 12.421~(9.75) \\
 & 1 & 2 & 0 & 0 & 2 & 11.592~(9.82) \\
 & 1 & 0.1 & 0.1 & 1e-5 & 0 & 11.316~(9.75) \\
 & 1 & 2 & 1 & 1e-5 & 2 & 9.826~(7.91) \\
\bottomrule
\end{tabular}
\end{table}

\begin{table}[htbp]
\centering
\caption{Ablation study for oscillating electron/plasma beam keeping $\gamma_\text{bwd}=\gamma_{con}=0$.}
\label{tab:abl_plasma}
\begin{tabular}{lcccccc}
\toprule
 & $\gamma_\text{id}$ & $\gamma_\text{fwd}$ & $\gamma_\text{bwd}$ & $\gamma_\text{con}$ & $\gamma_\text{tc}$ & Average relative error$(\%)$ \\
\midrule
\multirow{4}{*}{$N_\text{train}=28$, clean} & 1 & 0.5 & 0 & 0 & 0 & 16.570~(4.22) \\
 & 1 & 2 & 0 & 0 & 1e-2 & 3.608~(1.01) \\
 & 1 & 2 & 0.5 & 1e-7 & 0 & 14.958~(2.15) \\
 & 1 & 1 & 0.5 & 1e-7 & 1e-2 & 4.003~(0.73)  \\
\midrule
\multirow{4}{*}{$N_\text{train}=28$, 30 dB noise} & 1 & 0.5 & 0 & 0 & 0 & 27.006~(6.20) \\
 & 1 & 1 & 0 & 0 & 1e-2 & 8.892~(1.82) \\
 & 1 & 0.5 & 2 & 1e-4 & 0 & 21.901~(8.28) \\
 & 1 & 0.5 & 2 & 1e-4 & 1e-2 & 6.575~(1.32) \\
\midrule
\multirow{4}{*}{$N_\text{train}=52$, clean} & 1 & 2 & 0 & 0 & 0 & 3.956~(0.87) \\
 & 1 & 2 & 0 & 0 & 1e-2 & 1.928~(5.19) \\
 & 1 & 2 & 1e-3 & 1e-5 & 0 & 3.370~(0.61) \\
 & 1 & 2 & 1e-3 & 1e-5 & 1e-2 & 1.489~(0.46) \\
\midrule
\multirow{4}{*}{$N_\text{train}=52$, 30 dB noise} & 1 & 2 & 0 & 0 & 0 & 4.781~(0.74) \\
 & 1 & 2 & 0 & 0 & 1e-2 & 2.919~(4.88) \\
 & 1 & 2 & 1e-2 & 1e-6 & 0 & 3.900~(0.64) \\
 & 1 & 2 & 1e-2 & 1e-6 & 1e-4 & 3.376~(0.52) \\
\bottomrule
\end{tabular}
\end{table}

\begin{table}[htbp]
\centering
\caption{Ablation study for flow past cylinder keeping $\gamma_\text{bwd}=\gamma_{con}=0$.}
\label{tab:abl_flow}
\begin{tabular}{lcccccc}
\toprule
 & $\gamma_\text{id}$ & $\gamma_\text{fwd}$ & $\gamma_\text{bwd}$ & $\gamma_\text{con}$ & \textcolor{blue}{$\gamma_\text{tc}$} & Average relative error$(\%)$ \\
\midrule
\multirow{4}{*}{$N_\text{train}=28$, clean} & 1 & 2 & 0 & 0 & 0 & 10.182~(2.24) \\
 & 1 & 2 & 0 & 0 & 0.1 & 4.663~(0.60) \\
 & 1 & 2 & 2 & 1e-7 & 0 & 7.153~(1.06) \\
 & 1 & 2 & 4 & 1e-7 & 0.1 & 3.808~(0.45) \\
\midrule
\multirow{4}{*}{$N_\text{train}=28$, 30 dB noise} & 1 & 1e-2 & 0 & 0 & 0 & 20.724~(3.60) \\
 & 1 & 0.1 & 0 & 0 & 0.1 & 16.115~(2.23) \\
 & 1 & 1e-2 & 0.1 & 1e-5 & 0 & 18.376~(5.58) \\
 & 1 & 0.1 & 0.1 & 1e-5 & 1e-1 & 13.311~(0.53) \\
\midrule
\multirow{4}{*}{$N_\text{train}=40$, clean} & 1 & 1 & 0 & 0 & 0 & 1.695~(0.18) \\
 & 1 & 2 & 0 & 0 & 1e-3 & 0.535~(0.07) \\
 & 1 & 1 & 1e-4 & 1e-6 & 0 & 1.616~(0.17) \\
 & 1 & 1 & 1e-6 & 1e-5 & 1e-3 & 0.874~(0.11) \\
\midrule
\multirow{4}{*}{$N_\text{train}=40$, 30 dB noise} & 1 & 4 & 0 & 0 & 0 & 10.775~(1.29) \\
 & 1 & 2 & 0 & 0 & 1e-3 & 8.463~(0.96) \\
 & 1 & 4 & 1e-4 & 1e-5 & 0 & 10.933~(1.24) \\
 & 1 & 4 & 1e-4 & 1e-6 & 1e-3 & 9.275~(0.94) \\
\bottomrule
\end{tabular}
\end{table}

\begin{table}[htbp]
\centering
\caption{Error (\%) comparison for tcKAE with different $k_{tm}$ for pendulum data.}
\label{tab:steps_tc_pend}
\begin{tabular}{cc|cc|cc|cc}
\toprule
\multicolumn{4}{c|}{$N_\text{train}=32$} & \multicolumn{4}{c}{$N_\text{train}=64$} \\
\cmidrule(lr){1-4}  \cmidrule(lr){5-8}
$k_{tm}$ & Clean & $k_{tm}$ & 30 dB SNR & $k_{tm}$ & Clean & $k_{tm}$ & 30 dB SNR \\
\midrule
4 & 24.484~(19.85) & 16 & 29.205~(25.86)  & 6 & 7.639~(6.78) & 6 &  20.552~(17.73) \\
6 & 18.218~(15.16) & 18 & 25.970~(23.05)  & 8 & 8.218~(6.89) & 8 &  10.953~(8.75) \\
8 & 12.561~(10.20) & 20 & 15.147~(10.92)  & 10 & 7.491~(6.36) & 10 & 9.826~(7.91) \\
10 & 13.991~(11.17) & 22 & 30.916~(27.82)  & 12 & 10.716~(9.17) & 12 & 18.635~(15.69) \\
12 & 19.282~(15.45) & 24 & 28.682~(23.91)  & 14 & 13.428~(10.97) & 14 & 16.050~(13.06) \\
\bottomrule
\end{tabular}
\end{table}

\begin{table}[htbp]
\centering
\caption{Error (\%) comparison for tcKAE with different $k_{tm}$ for oscillating electron beam}
\label{tab:steps_tc_osc}
\begin{tabular}{cc|cc|cc|cc}
\toprule
\multicolumn{4}{c|}{$N_\text{train}=28$} & \multicolumn{4}{c}{$N_\text{train}=52$} \\
\cmidrule(lr){1-4} \cmidrule(lr){5-8}
$k_{tm}$ & Clean & $k_{tm}$ & 30 dB SNR & $k_{tm}$ & Clean & $k_{tm}$ & 30 dB SNR \\
\midrule
8 & 4.686~(0.87) & 6 & 7.573~(1.51)  & 20 & 1.577~(0.46) & 24 &  3.835~(0.63) \\
10 & 4.258~(0.75) & 8 & 7.004~(1.35)  & 22 & 1.539~(0.46) &26 &  3.855~(0.65) \\
12 & 4.003~(0.73) & 10 & 6.575~(1.23)  & 24 & 1.489~(0.46) & 28 & 3.376~(0.52) \\
14 & 4.057~(0.76) & 12 & 6.279~(1.19)  & 26 & 1.510~(0.48) & 30 & 3.399~(0.55) \\
16 & 3.888~(0.73) & 14 & 5.887~(1.09)  & 28 & 1.491~(0.48) & 32 & 3.919~(0.61) \\
\bottomrule
\end{tabular}
\end{table}

\begin{table}[htbp]
\centering
\caption{Error (\%) comparison for tcKAE with different $k_{tm}$ for flow past cylinder}
\label{tab:steps_tc_flow}
\begin{tabular}{cc|cc|cc|cc}
\toprule
\multicolumn{4}{c|}{$N_\text{train}=28$} & \multicolumn{4}{c}{$N_\text{train}=40$} \\
\cmidrule(lr){1-4} \cmidrule(lr){5-8}
$k_{tm}$ & Clean & $k_{tm}$ & 30 dB SNR & $k_{tm}$ & Clean & $k_{tm}$ & 30 dB SNR \\
\midrule
6 & 4.587~(0.46) & 6 & 14.763~(0.85)  & 26 & 1.108~(0.14) & 46 &  9.272~(0.98) \\
8 & 4.693~(0.44) & 8 & 14.494~(0.56)  & 28 & 1.000~(0.14) & 48 &  9.277~(0.98) \\
10 & 3.808~(0.45) & 10 & 13.311~(0.53)  & 30 & 0.874~(0.11) & 50 & 9.275~(0.94) \\
12 & 4.020~(0.36) & 12 & 13.220~(0.52)  & 32 & 0.943~(0.12) & 52 & 9.327~(0.97) \\
14 & 5.027~(0.49) & 14 & 12.933~(0.55)  & 34 & 0.857~(0.11) & 54 & 9.230~(0.96) \\
\bottomrule
\end{tabular}
\end{table}

\section*{Loss vs. Epochs}
For each system, we present two cases: one with clean data and smallest of $N_{\text{train}}$, and the other with noisy data and high $N_{\text{train}}$. Figure~\ref{fig:loss_pend} presents the plots for the pendulum, Figure~\ref{fig:loss_osc} presents it for the oscillating electron beam, and figure~\ref{fig:loss_flow} shows the plots for the flow past cylinder.

\begin{figure} [htbp]

    \centering
  \subfloat[$N_\text{train}=32$, Clean. \label{fig:loss_pend_32ntr_clean} ]{%
       \includegraphics[width=0.49\linewidth]{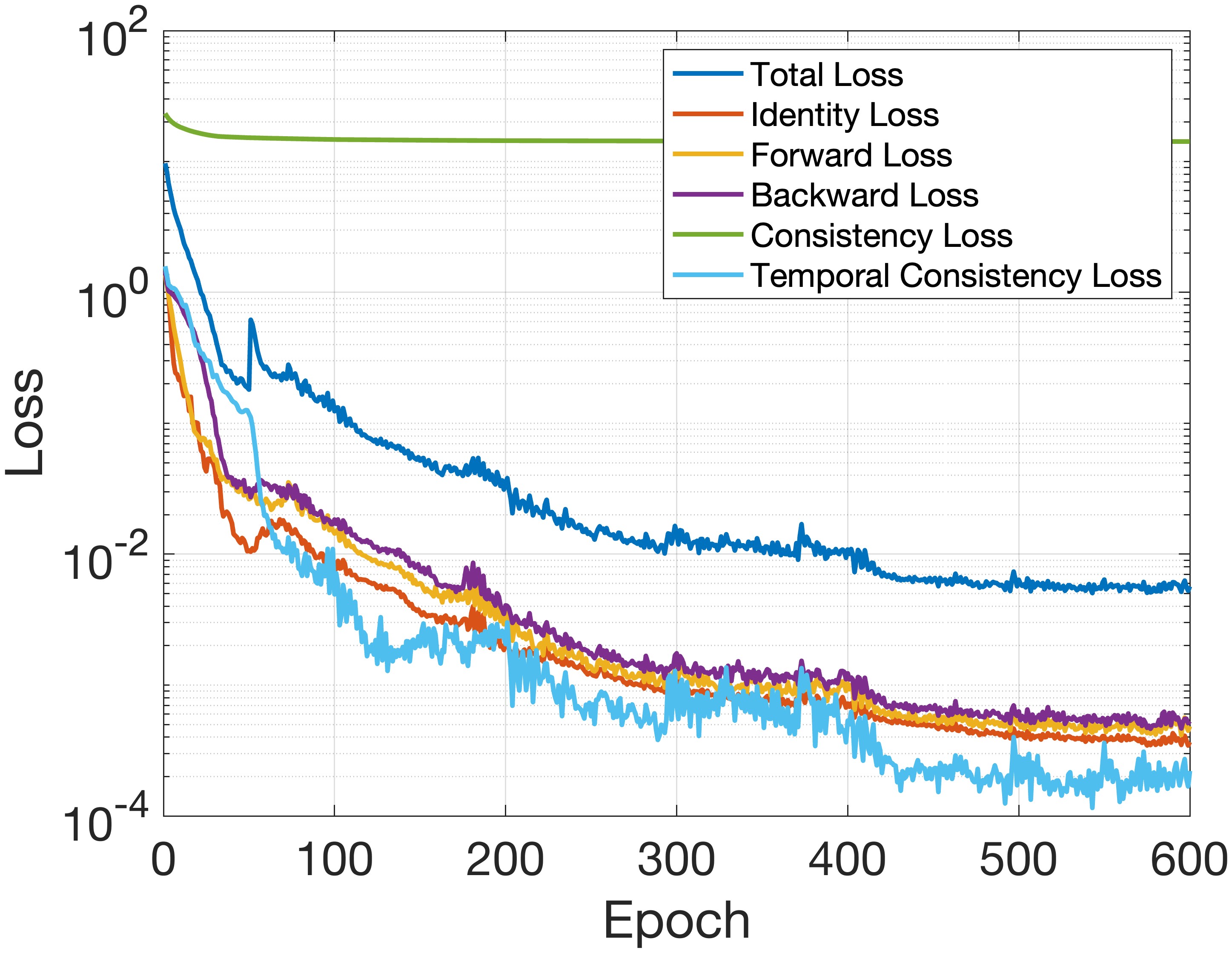}}
    \hfill
  \subfloat[ $N_\text{train}=64$, 30 dB SNR. \label{fig:loss_pend_ntr64_30SNR} ]{%
        \includegraphics[width=0.49\linewidth]{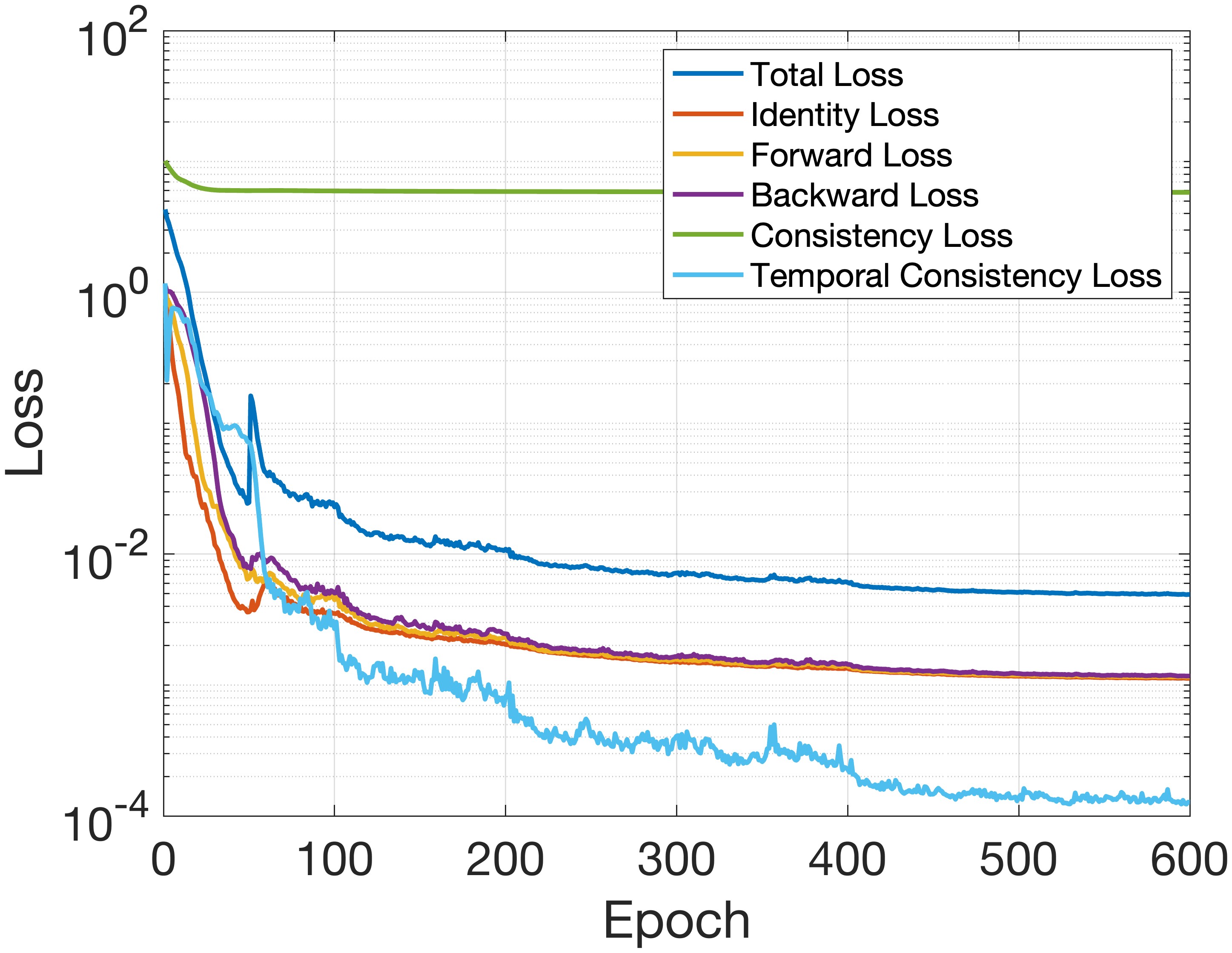}}

    \caption{\small{Loss vs. epochs plots for two pendulum data examples.} \label{fig:loss_pend} }
\end{figure}

\begin{figure} [htbp]

    \centering
  \subfloat[$N_\text{train}=28$, Clean. \label{fig:loss_wavy_ntr28+clean} ]{%
       \includegraphics[width=0.49\linewidth]{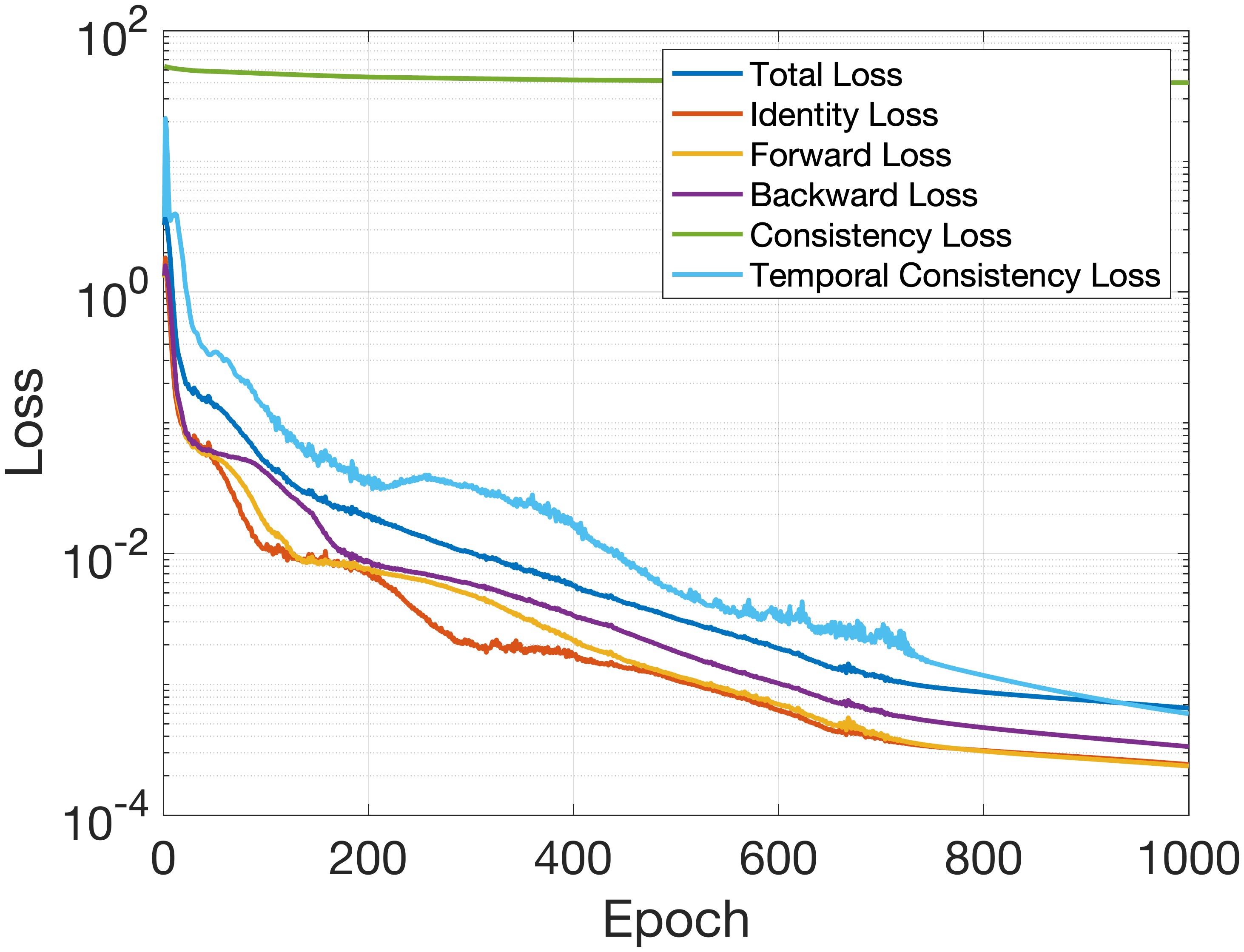}}
    \hfill
  \subfloat[ $N_\text{train}=52$, 30 dB SNR. \label{fig:loss-wavy_ntr52_30SNR} ]{%
        \includegraphics[width=0.49\linewidth]{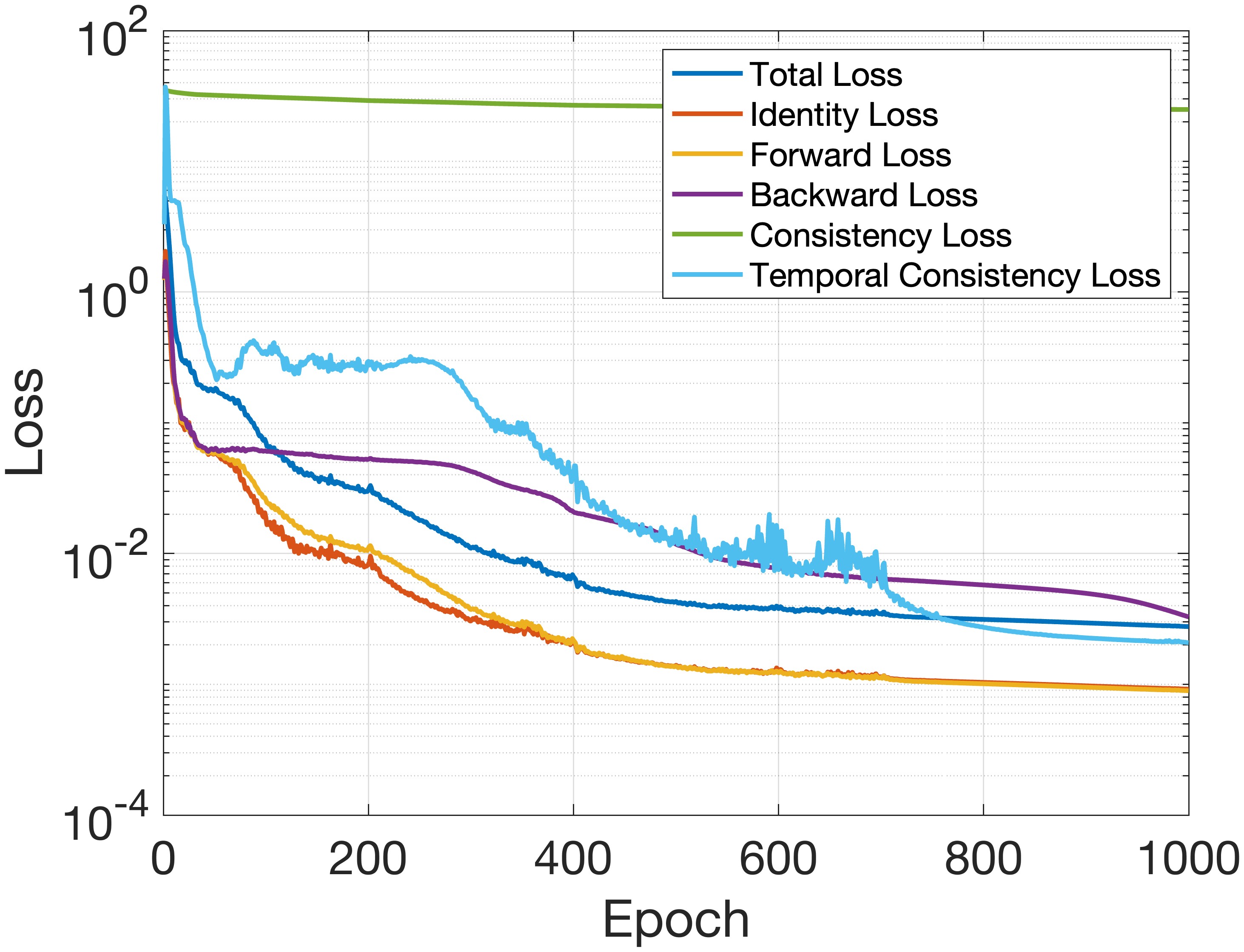}}

    \caption{\small{Loss vs. epochs plots for two oscillating electron beam data examples.}\label{fig:loss_osc}  }
\end{figure}

\begin{figure} [htbp]

    \centering
  \subfloat[$N_\text{train}=28$, Clean. \label{fig:loss_flow_ntr28_clean} ]{%
       \includegraphics[width=0.49\linewidth]{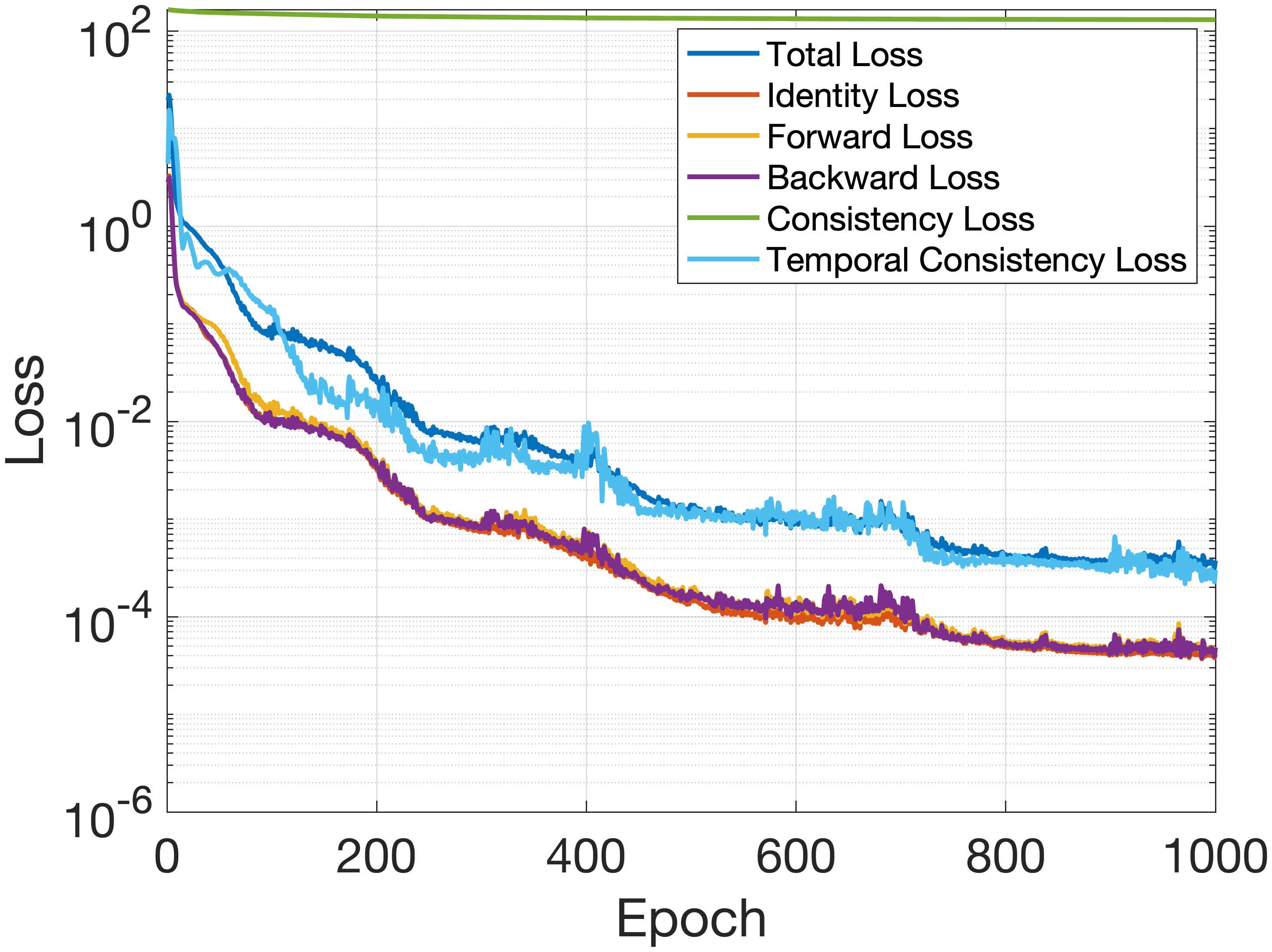}}
    \hfill
  \subfloat[ $N_\text{train}=40$, 30 dB SNR. \label{fig:loss_flow_ntr40_30SNR} ]{%
        \includegraphics[width=0.49\linewidth]{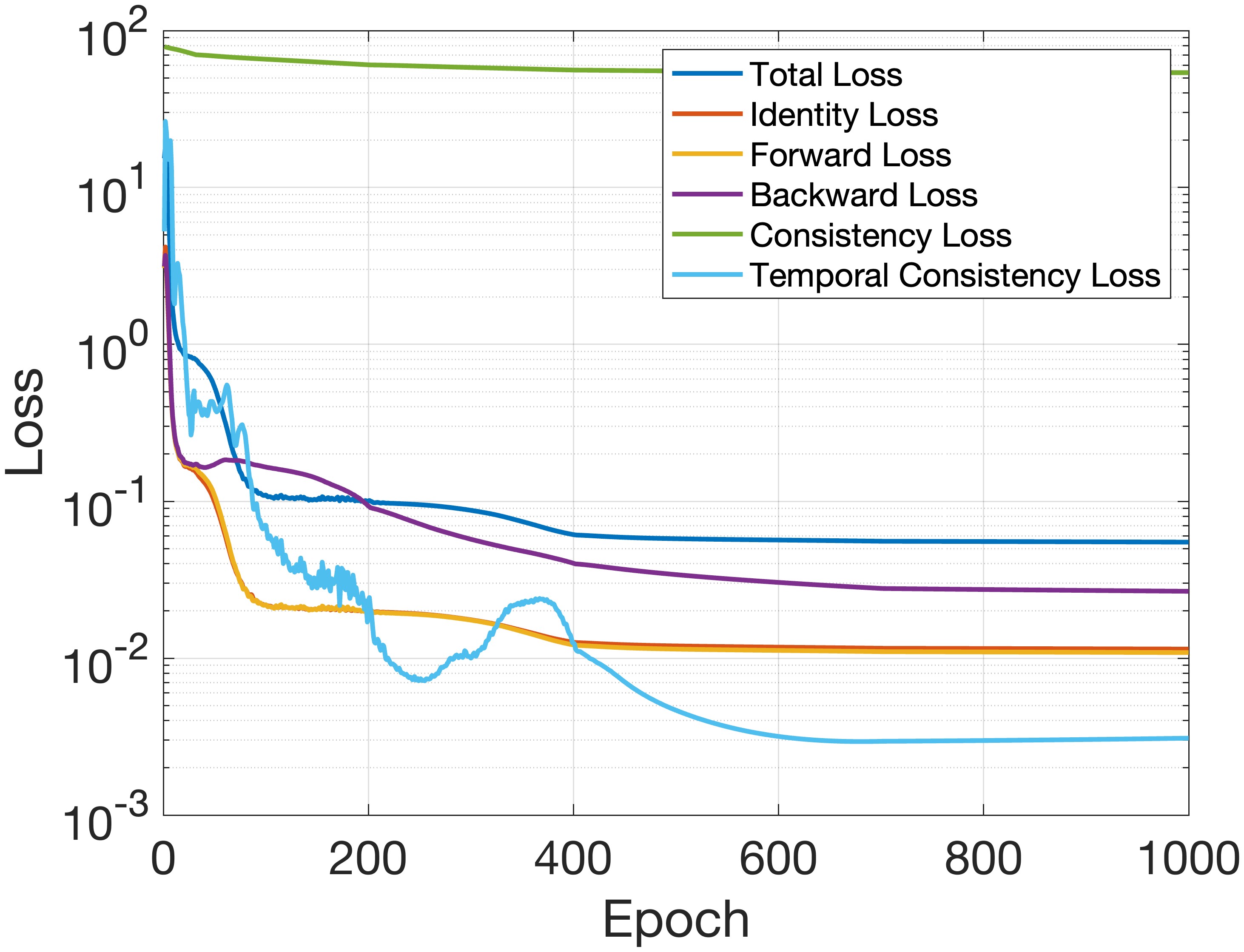}}

    \caption{\small{Loss vs. epochs plots for two flow past cylinder data examples.} \label{fig:loss_flow} }
\end{figure}

\section*{Training Time}
The training time for the DAE, cKAE and tcKAE for each system with clean and noisy data and different $N_{\text{train}}$ levels are presented here. Table~\ref{tab:runtime_pend} presents the training time data for the pendulum, Table~\ref{tab:runtime_osc} for the oscillating electron beam, and Table~\ref{tab:runtime_flow} for the flow past cylinder.

\begin{table}[htbp]
\centering
\caption{Training time (in minutes) for different $N_\text{train}$ and noise levels for pendulum data.}
\label{tab:runtime_pend}
\begin{tabular}{lcc|cc}
\toprule
& \multicolumn{2}{c|}{$N_\text{train}=32$} & \multicolumn{2}{c}{$N_\text{train}=64$} \\
\cmidrule(lr){2-3} \cmidrule(lr){4-5}
Method & Clean & 30 dB SNR & Clean & 30 dB SNR \\
\midrule
DAE & 2.540 & 2.527  & 5.399 & 6.022 \\
cKAE & 2.611 & 2.499  & 6.377 & 6.219 \\
tcKAE & 2.929 & 5.697  & 6.421 & 6.572 \\
\bottomrule
\end{tabular}
\end{table}

\newpage
\begin{table}[htbp]
\centering
\caption{Training time (in minutes) for different $N_\text{train}$ and noise levels for oscillating electron beam data.}
\label{tab:runtime_osc}
\begin{tabular}{lcc|cc}
\toprule
& \multicolumn{2}{c|}{$N_\text{train}=28$} & \multicolumn{2}{c}{$N_\text{train}=52$} \\
\cmidrule(lr){2-3} \cmidrule(lr){4-5}
Method & Clean & 30 dB SNR & Clean & 30 dB SNR \\
\midrule
DAE & 6.311 & 6.341  & 9.239 & 10.320 \\
cKAE & 6.362 & 6.323  & 10.737 & 9.111 \\
tcKAE & 8.366 & 8.032  & 16.779 & 19.351 \\
\bottomrule
\end{tabular}
\end{table}

\begin{table}[htbp]
\centering
\caption{Training time (in minutes) for different $N_\text{train}$ and noise levels for flow past cylinder data.}
\label{tab:runtime_flow}
\begin{tabular}{lcc|cc}
\toprule
& \multicolumn{2}{c|}{$N_\text{train}=28$} & \multicolumn{2}{c}{$N_\text{train}=52$} \\
\cmidrule(lr){2-3} \cmidrule(lr){4-5}
Method & Clean & 30 dB SNR & Clean & 30 dB SNR \\
\midrule
DAE & 8.329 & 7.806  & 8.231 & 8.805 \\
cKAE & 8.864 & 6.019 & 8.703 & 9.005 \\
tcKAE & 10.948 & 8.826 & 19.059 & 40.845 \\
\bottomrule
\end{tabular}
\end{table}
